\PassOptionsToPackage{unicode,colorlinks}{hyperref}
\PassOptionsToPackage{hyphens}{url}
\PassOptionsToPackage{dvipsnames}{xcolor}

\documentclass[10pt]{article} 
\usepackage[preprint]{tmlr}

\usepackage[utf8]{inputenc} 
\usepackage{hyperref}       
\usepackage{url}            
\usepackage{booktabs}       
\usepackage{nicefrac}       
\usepackage{microtype}      
\usepackage{xcolor}         

\usepackage{graphicx}

\usepackage{amssymb,amsmath,amsthm,mathtools}
\usepackage{upquote}
\usepackage{booktabs}
\usepackage{bm}

\UseMicrotypeSet[protrusion]{basicmath} 

\usepackage{caption}
\usepackage{subfigure}
\usepackage{wrapfig}

\usepackage{siunitx}

\definecolor{mydarkblue}{rgb}{0,0.08,0.45}
\hypersetup{ %
  colorlinks=true,
  linkcolor=mydarkblue,
  citecolor=mydarkblue,
  filecolor=mydarkblue,
  urlcolor=mydarkblue,
}

\usepackage[capitalize,noabbrev]{cleveref}

\usepackage{crossreftools}
\DeclareMathOperator{\E}{E}
\DeclareMathOperator{\var}{Var}
\DeclareMathOperator{\Var}{Var}

\DeclareMathOperator{\cor}{Corr}

\DeclareMathOperator{\diag}{diag}

\DeclareMathOperator{\sign}{sign}
\DeclareMathOperator{\st}{S}
\DeclareMathOperator{\normal}{Normal}
\DeclareMathOperator{\fnormal}{FoldedNormal}
\DeclareMathOperator{\bernoulli}{Bernoulli}
\DeclareMathOperator{\erf}{erf}

\DeclareMathOperator{\pdf}{\phi}
\DeclareMathOperator{\cdf}{\Phi}
\renewcommand{\vec}{\bm}
\newcommand{\mat}{\bm}
\newcommand*\du{\mathop{}\!\mathrm{d}}
\def\ceil#1{\lceil #1 \rceil}
\newcommand{\T}{\intercal}
\newcommand{\ones}{\bm{1}}

\theoremstyle{plain}
\newtheorem{theorem}{Theorem}[section]

\newtheorem{corollary}[theorem]{Corollary}
\theoremstyle{definition}
\newtheorem{definition}{Definition}

\theoremstyle{remark}

\newcommand{\pkg}[1]{\textsf{#1}}
\newcommand{\data}[1]{\texttt{#1}}

\newrobustcmd{\best}{\bfseries}

\usepackage{todonotes}

\graphicspath{{../plots/}}

\title{The Choice of Normalization Influences Shrinkage in Regularized Regression}

\author{%
  \name Johan Larsson \email jola@math.ku.dk\\
  \addr Department of Mathematical Sciences, University of Copenhagen\\
  \addr Department of Statistics, Lund University
  \AND
  \name Jonas Wallin \email jonas.wallin@stat.lu.se\\
  Department of Statistics, Lund University
}

\def\month{MM}  
\def\year{YYYY} 
\def\openreview{\url{https://openreview.net/forum?id=XXXX}} 

\begin{document}

\maketitle

\begin{abstract}
  Regularized models are often sensitive to the scales of the features in the data and it has
therefore become standard practice to normalize (center and scale) the features before
fitting the model. But there are many different ways to normalize the features and the
choice may have dramatic effects on the resulting model. In spite of this, there has so far
been no research on this topic. In this paper, we begin to bridge this knowledge gap by
studying normalization in the context of lasso, ridge, and elastic net regression. We focus
on binary features and show that their class balances (proportions of ones) directly
influences the regression coefficients and that this effect depends on the combination of
normalization and regularization methods used. We demonstrate that this effect can be
mitigated by scaling binary features with their variance in the case of the lasso and
standard deviation in the case of ridge regression, but that this comes at the cost of
increased variance of the coefficient estimates. For the elastic net, we show that scaling
the penalty weights, rather than the features, can achieve the same effect. Finally, we
also tackle mixes of binary and normal features as well as interactions and provide some
initial results on how to normalize features in these cases.

\end{abstract}

\section{Introduction}

When modeling high-dimensional data where the number of features~(\(p\)) exceeds the number
of observations~(\(n\)), it is impossible to apply classical statistical models such as
standard linear regression since the design matrix \(\mat X\) is no longer of full rank. A
common remedy to this problem is to \emph{regularize} the model by adding a penalty term to
the objective that punishes models with large coefficients. The resulting problem takes the
following form:
\begin{equation}
  \label{eq:general-objective}
  \operatorname*{minimize}_{\beta_0 \in \mathbb{R},\vec{\beta} \in \mathbb{R}^p} g(\beta_0, \vec{\beta}; \mat X, \vec y) + h(\vec\beta),
\end{equation}
where \(\vec y\) is the response vector, \(\mat X\) the design matrix, \(\beta_0\) the
intercept, and \(\bm{\beta}\) the coefficients. Furthermore, \(g\) is a data-fitting
function that attempts to optimize the fit to the data and \(h\) is a penalty that depends
only on \(\bm{\beta}\). Two common penalties are the \(\ell_1\) norm and squared \(\ell_2\)
norm penalties, which if \(g\) is the standard ordinary least-squares objective, represent
the lasso~\citep{tibshirani1996,santosa1986,donoho1994} and ridge (Tikhonov) regression
respectively.

These penalties depend on the magnitudes of the coefficients, which means that they are
sensitive to the scales of the features in \(\mat X\). To avoid this, it is common to
\emph{normalize} the features before fitting the model by shifting and scaling each feature
by some measures of their locations and scales, respectively. For some problems it is
possible to arrive at these measures by contextual knowledge of the data at hand. In most
cases, however, they must be estimated. A popular strategy is to use the mean and standard
deviation of each feature as location and scale factors respectively, which is called
\emph{standardization}.

The choice of normalization may, however, have consequences for the estimated model. As a
first example of this, consider \Cref{fig:realdata-paths}, which displays the
regularization paths for the lasso\footnote{The estimated coefficients of the lasso as the
  penalty strength is varied from a large-enough value for all coefficients to be zero to a
  low value at which the model is almost saturated.} for four datasets:
\data{housing}~\citep{harrison1978}, \data{a1a}~\citep{becker1996,platt1998},
\data{triazines}~\citep{king1995,hirst1994}, and \data{w1a}~\citep{platt1998}, and two
types of normalization.

In the figure, we have colored the lines corresponding to features that were among the
first five to enter the model in either of the two type of normalization schemes. Note that
the choice of normalization result in different sets of features being selected as well as
different coefficient estimates. This is especially striking in the case of
\data{triazines} and \data{w1a}.

\begin{figure}[bpt]
  \centering
  \includegraphics[]{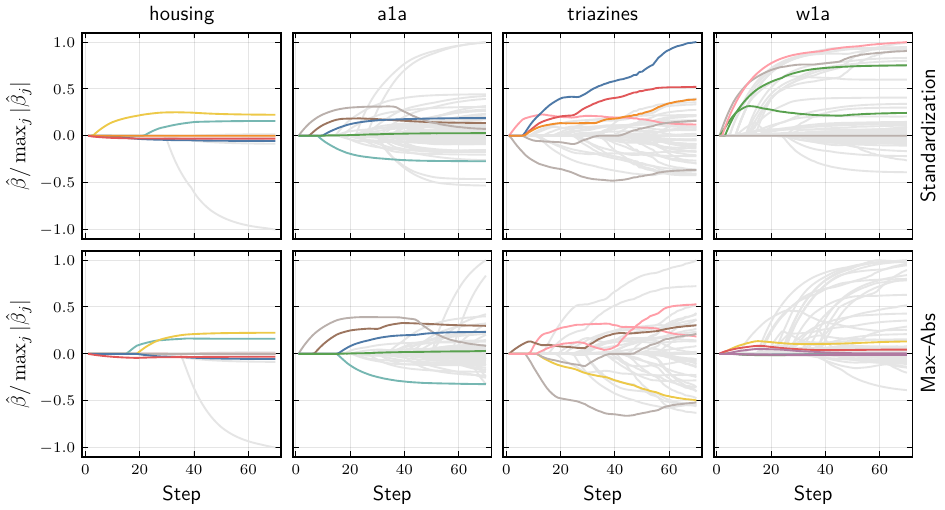}
  \caption{%
    Lasso paths for real datasets using two types of normalization:
    standardization and maximum absolute value normalization (max--abs). For
    each dataset, we have colored the coefficients if they were among the first
    five to become non-zero under either of the two normalization schemes.
    The \(x\)-axis shows the steps along the regularization path and
    the \(y\)-axis the estimated coefficients normalized by the
    maximum magnitude of the coefficients in each case. See
    \Cref{sec:data-summary} for more information about datasets used here.
  }
  \label{fig:realdata-paths}
\end{figure}

To illustrate this further, we show the estimated coefficients for the same datasets after
having fitted the lasso with a penalty strength (\(\lambda\)) set by 5-fold
cross-validation repeated 5 times on a 50\% training data subset. The estimated
coefficients for the 50\% held-out test data are shown in \Cref{tab:realdata-cv-coefs}. We
see that the estimated coefficients on the test set are different between the two cases,
especially for the \data{triazines} datasets, where the two normalization schemes disagree
completely.

\begin{table}[htpb]
  \centering
  \caption{
    Estimated lasso coefficients on test sets, with \(\lambda\) set from 5-fold
    cross-validation repeated 5 times. \(\hat{\bm{\beta}}_\text{std}\) and
    \(\hat{\bm{\beta}}_\text{max--abs}\) show the
    coefficients based on normalizing the design matrix with standardization and
    maximum absolute value (max--abs) normalization respectively. We show the five largest
    coefficients (in magnitude) for the standardization case and the respective
    coefficients for the max--abs case. In each case, we present the coefficients
    on the original scale of the features.
    See \Cref{sec:data-summary} for more
    information about these datasets.
  }
  \label{tab:realdata-cv-coefs}
  \begin{tabular}{
      S[table-format=-1.2,round-mode=figures,round-precision=2]
      S[table-format=-1.2,round-mode=figures,round-precision=2]
      S[table-format=-1.2,round-mode=figures,round-precision=2]
      S[table-format=-1.4,round-mode=figures,round-precision=2]
      S[table-format=1.3,round-mode=figures,round-precision=2]
      S[table-format=1.1,round-mode=figures,round-precision=2]
      S[table-format=1.1,round-mode=figures,round-precision=2]
      S[table-format=1.2,round-mode=figures,round-precision=2]
    }
    \toprule
    \multicolumn{2}{c}{\data{housing}} & \multicolumn{2}{c}{\data{a1a}}         & \multicolumn{2}{c}{\data{triazines}} & \multicolumn{2}{c}{\data{w1a}}                                                                                                                                                                   \\
    \cmidrule(rl){1-2}
    \cmidrule(rl){3-4}
    \cmidrule(rl){5-6}
    \cmidrule(rl){7-8}
    {\(\hat{\bm{\beta}}_\text{std}\)}  & {\(\hat{\bm{\beta}}_\text{max--abs}\)} & {\(\hat{\bm{\beta}}_\text{std}\)}    & {\(\hat{\bm{\beta}}_\text{max--abs}\)} & {\(\hat{\bm{\beta}}_\text{std}\)} & {\(\hat{\bm{\beta}}_\text{max--abs}\)} & {\(\hat{\bm{\beta}}_\text{std}\)} & {\(\hat{\bm{\beta}}_\text{max--abs}\)} \\
    \midrule
    -0.6309                            & -0.675004                              & 0.542532                             & 0.536762                               & 0.17369                           & 0.0                                    & 1.80033                           & 0.0                                    \\
    -1.38393                           & -0.779929                              & 0.327537                             & 0.523057                               & 0.0691224                         & 0.0                                    & 1.78789                           & 0.78455                                \\
    0.265659                           & 0.0                                    & -0.388699                            & -0.514664                              & 0.0284156                         & 0.0                                    & 1.78348                           & 0.630671                               \\
    -0.987113                          & -0.335828                              & 0.308715                             & 0.320916                               & 0.0706408                         & 0.0                                    & 1.44741                           & 0.07977                                \\
    2.77031                            & 3.06118                                & 0.175958                             & 0.23103                                & 0.0293775                         & 0.0                                    & 1.65553                           & 0.0                                    \\
    \bottomrule
  \end{tabular}
\end{table}

In spite of this apparent connection between normalization and regularization, there has so
far been almost no research on the topic. And in its absence, the choice of normalization
is typically motivated by computational concerns or by being ``standard''. This is
problematic since the effects of normalization are unknown and because there exists no
natural choice for many types of data. In particular, there is no obvious choice for binary
features (where each observation takes either of two values). In this paper we begin to
bridge this knowledge gap by studying normalization in the context of three particular
cases of the regularized problem in \Cref{eq:general-objective}: the lasso, ridge, and
elastic net~\citep{zou2005}. The latter of these, the elastic net, is a generalization of
the previous two, and is represented by the following optimization problem:
\begin{equation}
  \label{eq:elastic-net}
  \operatorname*{minimize}_{\beta_0 \in \mathbb{R},\vec{\beta} \in \mathbb{R}^p} \frac{1}{2} \lVert \vec y - \beta_0 - \tilde{\mat{X}}\vec{\beta} \rVert^2_2  + \lambda_1 \lVert \vec\beta \rVert_1 + \frac{\lambda_2}{2}\lVert \vec \beta \rVert_2^2,
\end{equation}
where setting \(\lambda_1 = 0\) results in ridge regression and setting \(\lambda_2 = 0\)
results in the lasso. Our focus in this paper is on binary data and we pay particular
attention to the case when they are imbalanced, that is, have relatively many ones or
zeroes. In this scenario, we demonstrate that the choice of normalization directly
influences the estimated coefficients and that this effect depends on the particular
combination of normalization and regularization.

Our key contributions are:
\begin{enumerate}
  \item We reveal that class balance in binary features directly affects lasso, ridge, and elastic
        net estimates, and show that scaling binary features with standard deviation (ridge) or
        variance (lasso) mitigates these effects at the cost of increased variance. Through
        extensive empirical analysis, we show that this finding extends to a wide range of
        settings~(\Cref{sec:experiments}).

  \item In our theoretical work, we examine this relationship in detail, showing that this bias
        from class imbalance holds even in the case of orthogonal
        features~(\Cref{sec:theory-binary-features}). For the elastic net, however, we show that
        this effect \emph{cannot} be mitigated by normalization and instead must be dealt with by
        scaling the penalty weights rather than the features~(\Cref{sec:binary-weighting}).

  \item For mixed data designs, we demonstrate how normalization choices implicitly determine the
        relative regularization effects on binary versus continuous
        features~(\Cref{sec:mixed-data}).

  \item For interaction features, we show that a common alternative to normalizing interaction
        features leads to biased estimates and provide an alternative approach that mitigates this
        problem~(\Cref{sec:interactions}).
\end{enumerate}

Collectively, our results demonstrate that normalization is not merely a preprocessing step
but rather an integral component of the model that requires careful consideration based on
data characteristics and the chosen regularization approach.

\section{Normalization}
To avoid possible confusion regarding the ambiguous use of terminology in the literature,
we will begin by clarifying what we mean by \emph{normalization}, which we define as the
process of centering and scaling the feature matrix.

\begin{definition}[Normalization]
  \label{def:normalization}
  Let \(\bm{X} \in \mathbb{R}^{n\times p}\) be the feature matrix and let
  \(\vec{c} \in \mathbb{R}^p\) and \(\vec{s} \in \mathbb{R}^p_+\) be centering
  and scaling factors respectively. Then \(\tilde{\bm{X}}\) is the
  \emph{normalized} feature matrix with elements given by
  \(\tilde{x}_{ij} = (x_{ij} - c_j)/s_j\).
\end{definition}

Some authors refer to the procedure in \Cref{def:normalization} as \emph{standardization},
but here we define standardization only as the case when centering with the mean and
scaling with the (uncorrected\footnote{Standard deviation computed without Bessel's
  correction (use of \(n-1\) instead of \(n\) in the standard deviation formula).}) standard
deviation.

There are many different normalization strategies and we have listed a few common choices
in \Cref{tab:normalization-types}. Standardization is perhaps the most popular type of
normalization, at least in the field of statistics. One of its benefits is that it
simplifies certain aspects of fitting the model, such as fitting the intercept. The
downside of standardization is that it involves centering by the mean, which destroys
sparsity in \(\bm{X}\) since centering shifts zero values to non-zero.

\begin{table}[t]
  \centering
  \caption{
    Common ways to normalize a matrix of features using centering and scaling
    factors \(c_j\) and \(s_j\), respectively. Note that \(\bar{x}_j\) is
    the arithmetic mean of feature \(j\) and that \(Q_a(\bm{x}_j)\) is the
    \(a\)th quartile of feature \(j\).
  }
  \label{tab:normalization-types}
  \begin{tabular}{lll}
    \toprule
    Normalization            & \(c_{j}\)          & \(s_j\)                                                     \\
    \midrule
    Standardization          & \(\bar{x}_j\)      & \(\frac{1}{\sqrt{n}} \lVert \vec{x}_j - \bar{x}_j\rVert_2\) \\
    \addlinespace
    \(\ell_1\)-Normalization & \(\bar{x}_j\)      & \(\frac{1}{\sqrt{n}} \lVert \vec{x}_j - \bar{x}_j\rVert_1\) \\
    \addlinespace
    Max--Abs                 & 0                  & \(\max_i|x_{ij}|\)                                          \\
    \addlinespace
    Min--Max                 & \(\min_i(x_{ij})\) & \(\max_i(x_{ij}) - \min_i(x_{ij})\)                         \\
    \addlinespace
    Robust Normalization     & \(Q_2(\bm{x}_j)\)  & \(Q_3(\bm{x}_j) - Q_1(\bm{x}_j)\)                           \\
    \addlinespace
    Adaptive Lasso           & 0                  & \(|\hat{\beta}_j^\text{OLS}|\)                              \\
    \bottomrule
  \end{tabular}
\end{table}

When \(\bm{X}\) is sparse, two common alternatives to standardization are min--max and
max--abs (maximum absolute value) normalization, which scale the data to lie in \([0, 1]\)
and \([-1, 1]\) respectively, and therefore retain sparsity when features are binary. These
methods are, however, both sensitive to outliers. And since sample extreme values often
depend on sample size, as in the case of normal data~(\Cref{sec:maxabs-theory}), use of
these methods may sometimes be problematic. Another alternative is to replace the
\(\ell_2\)-norm in the standardization method with the \(\ell_1\)-norm, which leads to
\(\ell_1\)-normalization. We note here that this method is often used without centering,
but this would make the method depend on the mean of the feature, as in the case of
max--abs normalization, and we therefore prefer the centered version here.

We have also included robust normalization in \Cref{tab:normalization-types}, which is a
version of normalization that uses the median and interquartile range (IQR) as centering
and scaling factors. Finally, we have also included the adaptive lasso~\citep{zou2006a},
which is a special case of normalization that fits a standard ordinary least-squares
regression (OLS) model to the data and uses the OLS estimates as scaling
factors.\footnote{In the case when \(p \gg n\), a ridge estimator is typically used
  instead.} We have both of these methods in \Cref{tab:normalization-types} for completeness,
but we will not study it further in this paper.

In the next section, we will examine how the choice of normalization affects the estimates
for the lasso, ridge, and elastic net regression.

\section{Ridge, Lasso, and Elastic Net Regression}%
\label{sec:theory}

In this setting we begin to describe the connection between normalization and the elastic
net estimator. We start by showing how the scaling and centering parameters of the
normalization method factors into the elastic net estimator in a general case. We then
narrow our focus to the case of binary features and present our main results on bias,
variance, and selection probability when the features are unbalanced.

Throughout the paper we assume that the response \(\vec{y}\) is generated according to
\(\bm{y} = \beta_0^* + \bm{X\beta}^* + \bm{\varepsilon}\), with \(\mat X\) being the \(n
\times p\) design matrix with features (columns) \(\vec x_j\), \(\bm{\varepsilon}\) the
vector of noise, with mean zero, finite variance \(\sigma_\varepsilon^2\), and identically
and independently distributed entries. We also assume \(\mat{X}\), \(\beta_0^*\), and
\(\vec{\beta}^*\) to be fixed and the features of the normalized design matrix to be
orthogonal, that is, \(\tilde{\mat{X}}^\intercal \tilde{\mat{X}} =
\diag\left(\tilde{\vec{x}}_1^\T \tilde{\vec{x}}_1, \dots, \tilde{\vec{x}}_p^\intercal
\tilde{\vec{x}}_p\right)\). In this case, it is a well-known fact~\citep{tibshirani1996}
that the solution to the elastic net problem is given by
\begin{equation}
  \label{eq:orthogonal-solution-normalized}
  \hat{\beta}^{(n)}_j = \frac{\st_{\lambda_1}\left(\tilde{\vec{x}}_j^\T \vec{y}\right)}{\tilde{\vec{x}}_j^\T \tilde{\vec{x}}_j + \lambda_2},
  \qquad
  \hat{\beta}_0^{(n)} = \frac{\vec{y}^\T \ones}{n},
\end{equation}
where \(\st_\lambda(z)\) is the soft-thresholding operator, defined as \(\st_\lambda(z) =
\sign(z) \max(|z| - \lambda, 0)\), which is the proximal operator of the \(\ell_1\) norm.
We refer to \Cref{sec:elastic-net-estimator} for a derivation of the results above.

The assumption of orthogonal features may seem strong and is indeed almost never realised
in practice. Here we use it for our theoretical results to show the direct connection
between normalization and the elastic net estimator, and prove that even in this simple
case, normalization has a pronounced effect on the estimates. In our experimental
work~(\Cref{sec:experiments}), however, we show that our findings extend to a much wider
class of designs and also refer the reader to \Cref{sec:orthogonality-assumption}, where we
discuss this assumption in detail and provide additional theoretical and empirical results
in relation to this.

Normalization changes the optimization problem and the estimated coefficients, which will
now be on the scale of the normalized features. But here we are interested in
\(\hat{\vec{\beta}}\): the coefficients on the scale of the original problem. To obtain
these, we transform the coefficients from the normalized problem, \(\hat\beta^{(n)}_j\),
back via \(\hat\beta_j = \hat\beta^{(n)}_j/s_j\) for \(j \in [p]\), where \([p] =
\{1,2,\dots,p\}\). There is a similar transformation for the intercept, but we omit it here
since we are not interested in interpreting it.

Taken together, this means that the solution for \(\hat{\vec{\beta}}\) can be expressed as
\[
  \hat{\beta}_j = \frac{\st_{\lambda_1}(\tilde{\vec{x}}_j^\T \vec{y})}{d_j}
\]
where
\begin{equation}
  \label{eq:z-d}
  \begin{aligned}
    \tilde{\vec{x}}_j^\T \vec{y} & = \frac{\beta_j^* n \nu_j- \vec{x}_j^\T \vec{\varepsilon}}{s_j}                \\
    d_j                          & = s_j\left(\frac{n \nu_j}{s_j^2} + \lambda_2\right) = s_j(\tilde{\vec{x}}_j^\T
    \tilde{\vec{x}}_j + \lambda_2),
  \end{aligned}
\end{equation}
with \(\nu_j\) being the uncorrected sample variance of \(\vec{x}_j\).
The bias and variance of \(\hat{\beta}_j\) are then given by
\begin{align}
  \E \hat\beta_j - \beta_j^* & = \frac{1}{d_j}\E \st_\lambda(\tilde{\vec{x}}_j^\T \vec{y}) - \beta^*_j,\label{eq:bias} \\
  \var \hat\beta_j           & = \frac{1}{d_j^2} \var \st_\lambda(\tilde{\vec{x}}_j^\T \vec{y}).\label{eq:variance}
\end{align}
See \Cref{sec:bias-var-deriv} for a derivation of the results above
as well as expressions for \(\E \st_\lambda(x)\) and \(\var S_\lambda(x)\).

These results hold in a general case. From now on, however, we will narrow our scope and
assume that the entries of \(\vec{\varepsilon}\) are identically, independently, and
normally distributed, in which case both the bias and variance of \(\hat{\beta}_j\) have
analytical expressions~(\Cref{sec:normally-distributed-noise}) and
\[
  \tilde{\vec{x}}_j^\T \vec{y} \sim \normal\left(\mu_j = \tilde{\vec{x}}_j^\T\vec{x}_j \beta_j^*, \sigma_j^2 = \tilde{\vec{x}}_j^\T\tilde{\vec{x}}_j \sigma_\varepsilon^2 \right).
\]
So far, we have assumed nothing about the features themselves, apart from being orthogonal
to each other. The main focus of our paper, however, is binary features, which we will now
turn to.

\subsection{Binary Features}%
\label{sec:theory-binary-features}

When \(x_{ij} \in \{0, 1\}\) for all \(i\), we define \(\bm{x}_j\) to be a \emph{binary
  feature}, and the \emph{class balance} of this feature as \(q_j =
\frac{1}{n}\sum_{i=1}^n{x_{ij}}\): the proportion of ones. It would make no difference to
the majority of our results if we were to swap the ones and zeros as long as an intercept
is included, and ``class balance'' is then equivalent to the proportion of either. But in
the case of interactions~(\Cref{sec:interactions}), the choice does in fact matter.

If feature $j$ is binary then \(\nu_j = (q_j - q_j^2)\) (the uncorrected sample variance
for a binary feature), which in \Cref{eq:z-d} yields
\begin{equation*}
  \tilde{\vec{x}}_j^\T \vec{y} = \frac{\beta_j^* n(q_j - q_j^2) - \vec{x}_j^\T \vec{\varepsilon}}{s_j}, \qquad
  d_j                          = s_j \left(\frac{n(q_j - q_j^2)}{s_j^2} + \lambda_2\right),
\end{equation*}
and consequently
\[
  \mu_j = \frac{\beta^*_j n(q_j - q_j^2)}{s_j}\quad \text{and} \quad \sigma_j^2 = \frac{\sigma_\varepsilon^2n(q_j- q_j^2)}{s^2_j}.
\]
We obtain bias and variance of the estimator with respect to \(q_j\) by inserting
\(\tilde{\vec{x}}_j^\T \vec{y}\) and \(d_j\) into \Cref{eq:bias,eq:variance}.

The presence of the factor \(q_j - q_j^2\) in \(\mu_j\), \(\sigma_j^2\), and \(d_j\)
indicates a link between class balance and the elastic net estimator and, moreover, that
this relationship is mediated by the scaling factor \(s_j\). To achieve some initial
intuition for this relationship, consider the noiseless case (\(\sigma_\varepsilon = 0\))
in which we have
\begin{equation}
  \label{eq:noiseless-estimator}
  \hat{\beta}_j = \frac{\st_{\lambda_1}(\tilde{\vec{x}}_j^\intercal \vec{y})}{s_j\left(\tilde{\vec{x}}_j^\intercal \tilde{\vec{x}}_j + \lambda_2\right)}
  =
  \frac{\st_{\lambda_1}\left(\frac{\beta_j^* n (q_j - q_j^2)}{s_j}\right)}{s_j\left(\frac{n(q_j - q_j^2)}{s_j^2} + \lambda_2\right)}.
\end{equation}
This expression shows that class balance (\(q_j\)) directly affects the estimator through
the factor \(q_j - q_j^2\) (the variance of the binary feature). Starting with the lasso
(\(\lambda_2 = 0\)), observe that the soft-thresholding part of the estimator (numerator)
diminishes for values of \(q_j\) close to \(0\) or \(1\) unless we use the scaling factor
\(s_j = q_j - q_j^2\), in which case \Cref{eq:noiseless-estimator} simplifies to
\[
  \hat{\beta}_j
  = \frac{\st_{\lambda_1}\left(\frac{\beta_j^* n (q_j - q_j^2)}{q_j - q_j^2}\right)}{(q_j - q_j^2)\left(\frac{n(q_j - q_j^2)}{(q_j - q_j^2)^2} + \lambda_2\right)}
  = \frac{\st_{\lambda_1}(\beta_j^* n)}{n},
\]
which is independent of \(q_j\). For other choices of \(s_j\), the soft-thresholding part
of the estimator will be affected by class balance and also depend on the size of
\(\lambda_1\), with larger values of \(\lambda_1\) leading to a larger effect of class
balance.

Turning to the ridge case (\(\lambda_1 = 0\)), the soft-thresholding part simplifies to the
identity function and the shrinkage instead comes from the presence of \(\lambda_2\) in the
denominator, which will scale the estimator towards zero as \(q_j\) approaches 0 or 1.
Contrary to the lasso case, when then instead need to take \(s_j = (q_j - q_j^2)^{1/2}\),
in which case the expression becomes
\[
  \hat{\beta}_j
  = \frac{\st_{\lambda_1}\left(\frac{\beta_j^* n (q_j - q_j^2)}{(q_j - q_j^2)^{1/2}}\right)}{(q_j - q_j^2)^{1/2}\left(\frac{n(q_j - q_j^2)}{(q_j - q_j^2)} + \lambda_2\right)}
  = \frac{\beta_j^* n}{n + \lambda_2},
\]
which is again independent of \(q_j\).

Observe, however, that for the elastic net (\(\lambda_1 > 0, \lambda_2 > 0\)), there exists
no \(s_j\) that can make the estimator independent of \(q_j\). In other words, there is no
type of normalization, at least under our parameterization, that is able to mitigate the
class balance bias in this case. In \Cref{sec:binary-weighting}, however, we will show how
to tackle this issue for the elastic net by scaling the penalty weights. But for now we
continue to study the case of normalization.

Based on the reasoning above, we will consider the scaling parameterization \(s_j =
(q_j-q_j^2)^\delta\), \(\delta \geq 0\), which includes the cases that we are primarily
interested in, namely \(\delta = 0\) (no scaling, as in min--max and max--abs
normalization), \(\delta = 1/2\) (standard-deviation scaling), and \(\delta = 1\) (variance
scaling). The last of these, variance scaling, is in fact equivalent to scaling with the
mean-centered \(\ell_1\)-norm in this particular case of binary features. In
\Cref{sec:noiseless-estimator}, we expand \Cref{eq:noiseless-estimator} under this
parameterization to clarify how the choice of \(\delta\) affects class-balance bias in the
lasso and ridge cases and why it is impossible (under our parameterization) to remove this
bias in the case of the elastic net.

Another consequence of \Cref{eq:noiseless-estimator}, which holds also in the noisy
situation, is that normalization affects the estimator even when the binary feature is
balanced (\(q_j = 1/2\)). \(\delta = 0\), for instance, scales \(\beta_j^*\) in the input
to \(\st_\lambda\) by \(n (q_j - q_j^2) = n/4\). \(\delta = 1\), in contrast, imposes no
such scaling in the class-balanced case. And for \(\delta = 1/2\), the scaling factor is
\(n/2\). Generalizing this, we see that to achieve equivalent scaling in the class-balanced
case for all types of normalization, under our parameterization, we would need to use \(s_j
= 4^{\delta - 1} (q_j - q_j^2)^\delta\). But this only resolves the issue for the lasso. To
achieve a similar effect for ridge regression, we would need another (but similar)
modification. When all features are binary, we can just scale \(\lambda_1\) and
\(\lambda_2\) to account for this effect,\footnote{We use this strategy in all of the
  following examples.} which is equivalent to modifying \(s_j\). But when we consider mixes
of binary and normal features in \Cref{sec:mixed-data}, we need to exert extra care.

We now proceed to consider how class balance affects the bias, variance, and selection
probability of the elastic net estimator under the presence of noise. A consequence of our
assumption of a normal error distribution and consequent normal distribution of
\(\tilde{\vec{x}}_j^\T \vec{y}\) is that the probability of selection in the elastic net
problem is given by
\begin{align}
  \label{eq:selection-probability}
  \Pr\left(\hat{\beta}_j \neq 0\right)  ={} & \cdf \left( \frac{\beta_j^*n (q_j-q_j^2)^{1/2} - \lambda_1(q_j-q_j^2)^{\delta - 1/2}}{\sigma_\varepsilon \sqrt{n}}\right)\nonumber \\
                                            & + \cdf \left( \frac{-\beta_j^*n (q_j-q_j^2)^{1/2} - \lambda_1(q_j-q_j^2)^{\delta - 1/2}}{\sigma_\varepsilon \sqrt{n}}\right).
\end{align}
where \(\cdf\) is the cumulative distribution function of the standard normal distribution.
Letting \(\theta_j = -\mu_j - \lambda_1 \) and \(\gamma_j = \mu_j - \lambda_1\), we can
express this probability asymptotically as \(q_j \rightarrow 1^-\) as
\begin{equation}
  \label{eq:selection-probability-limit}
  \lim_{q_j \rightarrow 1^-} \Pr(\hat{\beta}_j \neq 0) =
  \begin{cases}
    0                                                                & \text{if } 0 \leq \delta < \frac{1}{2}, \\
    2\cdf\left(-\frac{\lambda_1}{\sigma_\varepsilon \sqrt{n}}\right) & \text{if } \delta = \frac{1}{2},        \\
    1                                                                & \text{if } \delta > \frac{1}{2}.
  \end{cases}
\end{equation}

In \Cref{fig:selection-probability}, we plot this probability for various settings of
\(\delta\) for a single feature. Our intuition from the noiseless case holds: suitable
choices of \(\delta\) can mitigate the influence of class imbalance on selection
probability. The lower the value of \(\delta\), the larger the effect of class imbalance
becomes. Note that the probability of selection initially decreases also in the case when
\(\delta \geq 1\). This is a consequence of increased variance of \(\tilde{\vec{x}}_j^\T
\vec{y}\) due to the scaling factor that inflates the noise term. But as \(q_j\) approaches
1, the probability eventually rises towards 1 for \(\delta \in \{1, 1.5\}\). The reason for
this is that this rise in variance eventually quells the soft-thresholding effect
altogether. Note, also, that the selection probability is unaffected by \(\lambda_2\).

\begin{figure}[htpb]
  \centering
  \includegraphics[]{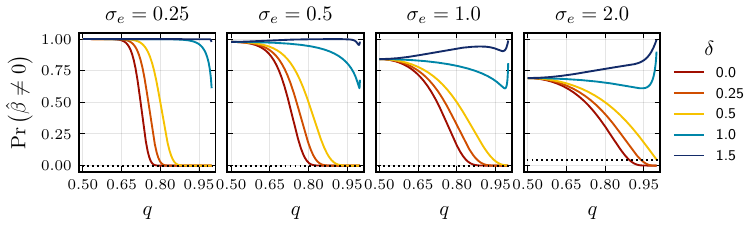}
  \caption{%
    Probability of selection in the lasso given measurement noise
    \(\sigma_\varepsilon\), regularization level \(\lambda_1\), and class
    balance \(q\). The scaling factor is set to \(s_j = (q - q^2)^\delta\),
    \(\delta \geq 0\). The dotted line represents the asymptotic limit for
    \(\delta = 1/2\) from \Cref{eq:selection-probability-limit}.
    \label{fig:selection-probability}}
\end{figure}

Now we turn to the impact of class imbalance on bias and variance of the elastic net
estimator. We begin, in \Cref{thm:classbalance-bias}, by considering the expected value of
the elastic net estimator in the limit as \(q_j \rightarrow 1^-\).

\begin{theorem}
  \label{thm:classbalance-bias}
  If \(\vec{x}_j\) is a binary feature with class balance \(q_j \in (0, 1)\),
  \(\lambda_1 \in [0,\infty)\), \(\lambda_2 \in [0,\infty)\),
  \(\sigma_\varepsilon > 0\), and \(s_j = (q_j - q_j^2)^{\delta}\), \(\delta
  \geq 0\)  then
  \[
    \lim_{q_j \rightarrow 1^-} \E \hat{\beta}_j =
    \begin{cases}
      0                                                                                                  & \text{if } 0 \leq \delta < \frac{1}{2}, \\
      \frac{2n \beta_j^*}{n + \lambda_2} \cdf\left(-\frac{\lambda_1}{\sigma_\varepsilon \sqrt{n}}\right) & \text{if } \delta = \frac{1}{2},        \\
      \beta^*_j                                                                                          & \text{if } \delta > \frac{1}{2}.
    \end{cases}
  \]
\end{theorem}

\Cref{thm:classbalance-bias} shows that bias of the elastic net estimator approaches
\(-\beta_j^*\) as \(q_j \rightarrow 1^-\) when \(0 \leq \delta < 1/2\). When \(\delta =
1/2\) (standardization), the estimate approaches a constant that depends on regularization
strength, noise level, and the true strength of the coefficient. For \(\delta > 1/2\), the
estimate is asymptotically unbiased as a by-product of variance dominating in the limit as
\(q_j \rightarrow 1^{-1}\), which suggests that variance-scaling (\(\ell_1\)-norm
normalization) could be problematic in a scenario with much noise and highly imbalanced
features.

In \Cref{thm:classbalance-variance}, we continue by studying the variance in the limit as
\(q_j \rightarrow 1^-\), which shows that the variance of the elastic net estimator tends
to \(\infty\) in the limit unless the scaling parameter \(s_j < 1/2\).

\begin{theorem}
  \label{thm:classbalance-variance}
  Assume the conditions of \Cref{thm:classbalance-bias} hold, except that
  \(\lambda_1 > 0\). Then
  \[
    \lim_{q_j \rightarrow 1^-} \var \hat{\beta}_j =
    \begin{cases}
      0      & \text{if } 0 \leq \delta < \frac{1}{2}, \\
      \infty & \text{if } \delta \geq \frac{1}{2}.
    \end{cases}
  \]
\end{theorem}

Note that \Cref{thm:classbalance-variance} applies only to the case when \(\lambda_1 > 1\).
In \Cref{cor:ridge-variance}~(\Cref{sec:ridge-variance}), we state the corresponding result
for ridge regression.

Taken together, \Cref{thm:classbalance-bias,thm:classbalance-variance,cor:ridge-variance},
indicate that the choice of scaling parameter constitutes a bias--variance trade-off with
respect to \(\delta\): increasing \(\delta\) reduces class-balance bias, but does so at the
cost of increased variance.

In \Cref{fig:bias-var-onedim-lasso}, we now visualize bias, variance, and mean-squared
error for ranges of class balance and various noise-level settings for a lasso problem. The
figure demonstrates the bias--variance trade-off that our asymptotic results suggest and
indicates that the optimal choice of \(\delta\) is related to the noise level in the data.
Since this level is typically unknown and can only be reliably estimated in the
low-dimensional setting, it suggests there might be value in selecting \(\delta\) through
hyper-parameter optimization.\footnote{In \Cref{sec:normalization-tuning} we demonstrate
  the utility of doing so.} In
\Cref{fig:bias-var-onedim-ridge-full}~(\Cref{sec:additional-results-biasvar}) we show
results for ridge regression as well. As expected, it is then \(\delta = 1/2\) that leads
to unbiased estimates.

\begin{figure}[htb]
  \centering
  \includegraphics[]{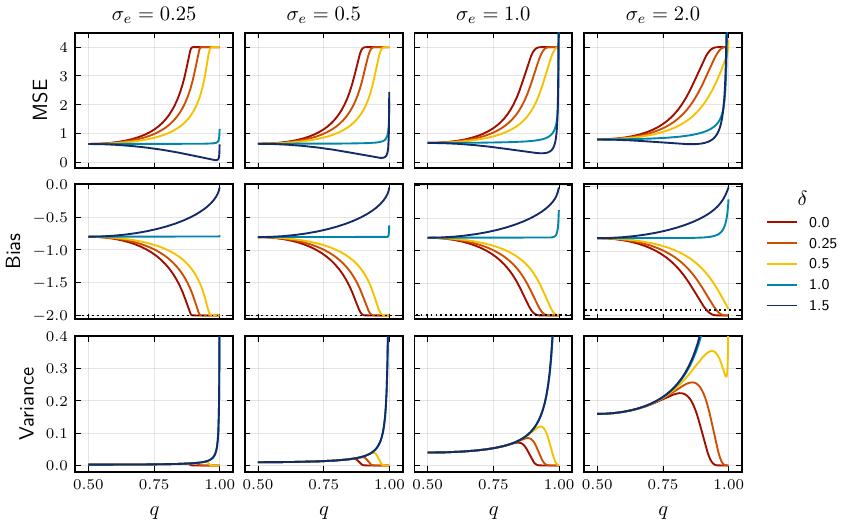}
  \caption{%
    Bias, variance, and mean-squared error for a one-dimensional lasso problem,
    parameterized by noise level (\(\sigma_\varepsilon\)), class balance (\(q\)), and
    scaling (\(\delta\)). Dotted lines represent asymptotic bias of the lasso
    estimator in the case when \(\delta = 1/2\).}
  \label{fig:bias-var-onedim-lasso}
\end{figure}

So far, we have only considered a single binary feature, but in
\Cref{sec:power-fdr-multiple} we present results on power and false discovery rates for
problems with multiple features. In the next section we will also step beyond the
all-binary context and consider mixes of binary and continuous features.

\subsection{Mixed Data}%
\label{sec:mixed-data}

A fundamental problem with mixes of binary and continuous features is deciding how to put
these features on the same scale in order to regularize each type of feature fairly. In
principle, we need to match a one-unit change in the binary feature with some amount of
change in the normal feature. This problem has previously been tackled, albeit from a
different angle, by \citet{gelman2008}, who argued that the common default choice of
presenting standardized regression coefficients unduly emphasizes coefficients from
continuous features.

To setup this situation formally, we will say that the effects of a binary feature
\(\vec{x}_1\) and a normal feature \(\vec{x}_2\) are \emph{comparable} if \(\beta^*_1 =
\kappa \sigma \beta^*_2\), where \(\kappa > 0\) represents the number of standard
deviations of the normal feature we consider to be comparable to one unit on the binary
feature. As an example, assume \(\kappa = 2\). Then, if \(\vec{x}_2\) is sampled from
\(\normal\left(\mu_j, \sigma^2 = (1/2)^2\right)\), the effects of \(\vec{x}_1\) and
\(\vec{x}_2\) are comparable if \(\beta_1^* = 2\sigma \beta_2^* = \beta_2^*\).

The definition above refers to \(\bm{\beta}^*\), but for our regularized estimates we need
\(\hat{\beta}_1 = \kappa\sigma\hat{\beta}_2\) to hold. If we assume that we are in a
noiseless situation (\(\sigma_\varepsilon = 0\)), are standardizing the normal feature, and
that, without loss of generality, \(\bar{x}_1 = 0\), then we need the following equality to
hold:
\begin{equation}
  \label{eq:comparable-effects}
  \hat{\beta}_1 = \kappa\sigma \hat{\beta}_2 \implies \frac{\st_{\lambda_1}(\tilde{\vec{x}}_1^\intercal \vec{y})}{s_1\left(\tilde{\vec{x}}_1^\intercal \tilde{\vec{x}}_1 + \lambda_2\right)}  =\frac{\kappa\sigma \st_{\lambda_1}(\tilde{\vec{x}}_2^\intercal \vec{y})}{s_2\left(\tilde{\vec{x}}_2^\intercal \tilde{\vec{x}}_2 + \lambda_2\right)}  \implies \frac{\st_{\lambda_1}\left(\frac{n\beta_1^* (q - q^2)}{s_1}\right)}{s_1\left(\frac{n(q - q^2)}{s_1^2} + \lambda_2\right)} = \frac{\kappa \st_{\lambda_1}\left(\frac{n\beta_1^*}{\kappa} \right)}{n + \lambda_2}.
\end{equation}
For the lasso (\(\lambda_2 = 0\)) and ridge regression (\(\lambda_1=0\)), we see that the
equation holds for \(s_1 = \kappa (q - q^2)\) and \(s_1 = (q - q^2)^{1/2}\), respectively.
In other words, we achieve comparability in the lasso by scaling each binary feature with
its variance times \(\kappa\). And for ridge regression, we can achieve comparability by
scaling with standard deviation, irrespective of \(\kappa\). For any other choices of
\(s_1\), equality holds only at a fixed level of class balance. Let this level be \(q_0\).
Then, to achieve equality for \(\lambda_2 = 0\), we need \(s_1 =\kappa (q_0 - q_0^2)^{1 -
  \delta}(q - q^2)^\delta\). Similarly, for \(\lambda_1 = 0\), we need \(s_1 = (q_0 -
q_0^2)^{1 - 2\delta} (q - q^2)^\delta\). In the sequel, we will assume that \(q_0 = 1/2\),
to have effects be equivalent for the class-balanced case.

Note that this means that the choice of normalization has an implicit effect on the
relative penalization of binary and normal features---even in the class-balanced case
(\(q_1 = 1/2\)). If we for instance use \(\delta=0\) and fit the lasso, then
\Cref{eq:comparable-effects} for a binary feature with \(q_1=1/2\) becomes
\(4\st_{\lambda_1}\left(n\beta_1^*/4\right) = \kappa \st_{\lambda_1}(n\beta_1^*/\kappa ),\)
which implies \(\kappa = 4\). In other words, the choice of normalization equips our model
with a belief about how binary and normal features should be penalized relative to one
another.

For the rest of this paper, we will use \(\kappa = 2\) and say that the effects are
comparable if the effect of a flip in the binary feature equals the effect of a
two-standard deviation change in the normal feature. We motivate this by an argument by
\citet{gelman2008}, but want to stress that the choice of \(\kappa\) should, if possible,
be based on contextual knowledge of the data and that our results depend only superficially
on this particular setting.

\subsection{Interactions}\label{sec:interactions}

The elastic net can be extended to include interactions. There is previous literature on
this topic~\citep{bien2013,lim2015,zemlianskaia2022}, but it has not considered the
possible influence of normalization. Here, we will consider simple pairwise interactions
with no restriction on the presence of main effects. For our analysis, we let \(\vec{x}_1\)
and \(\vec{x}_2\) be two features of the data and \(\bm{x}_3\) their interaction, so that
\(\beta_3\) represents the interaction effect.

We consider two cases in which we assume that the features are orthogonal and that
\(\vec{x}_1\) is binary with class balance \(q_1\). In the first case, we let \(\bm{x}_2\)
be normal with mean \(\mu\) and variance \(\sigma^2\), and in the second case \(\bm{x}_2\)
be binary with class balance \(q_2\). To construct the interaction feature, we
center\footnote{See \Cref{sec:centering-interactions} for motivation for why we center the
  features before computing the interaction.} the main features and then multiply
element-wise. The elements of the interaction feature are then given by \(x_{3,i} =
(x_{1,i} - \bar{\bm{x}}_1)(x_{2,i} - \bar{\bm{x}}_2)\).

If \(\bm{x}_2\) is normal and both features are centered before computing the interaction
term, the variance becomes \(\sigma^2 (q-q^2)\), which suggests using \(s_3 = \sigma (q -
q^2)^\delta\) along the lines of our previous reasoning. And if \(\bm{x}_2\) is binary,
instead, then similar reasoning suggests using \(s_3 = ((q_1-q_1^2)(q_2-q_2^2))^\delta\).
In \Cref{sec:experiments-interactions}, we study the effects of these choices in simulated
experiments.

\subsection{The Weighted Elastic Net}\label{sec:binary-weighting}

We have so far shown that certain choices of normalization can mitigate the class-balance
bias imposed by the lasso and ridge regularization. But we have also
demonstrated~(\Cref{sec:theory-binary-features}) that there is no (simple) choice of
scaling that can achieve the same effect for the elastic net.
\Cref{eq:noiseless-estimator}, however, suggests a natural alternative to normalization,
which is to use the weighted elastic net, in which we minimize
\[
  \frac{1}{2} \lVert \vec{y} - \beta_0 - \mat{X}\vec{\beta}\rVert_2^2 + \lambda_1 \sum_{j=1}^p u_j |\beta_j| + \frac{\lambda_2}{2} \sum_{j=1}^p v_j \beta_j^2,
\]
with \(\vec{u}\) and \(\vec{v}\) being \(p\)-length vectors of positive scaling factors.
This is equivalent to the standard elastic net for a normalized feature matrix when \(u_j =
s_j\) and \(v_j = s_j^2\), which can be seen by substituting \(\beta_js_j =
\tilde{\beta}_j\) in \Cref{eq:elastic-net} and solving for \(\tilde{\vec{\beta}}\). Note
that we do not need to rescale the coefficients from this problem as we would for the
standard elastic net on normalized data.

This allows us to control class-balance bias by setting our weights according to \(u_j =
v_j = (q_j - q_j^2)^{\omega}\) and counteract it, at least in the noiseless case, with
\(\omega = 1\), which, we want to emphasize, is \emph{not} possible using the standard
elastic net. For the lasso and ridge regression, however, this setting of \(\omega=1\) is
equivalent to using \(\delta = 1\) and \(\delta = 1/2\), respectively, in the standard
elastic net with normalized data. Results analogous to those in
\Cref{sec:theory-binary-features} can be attained with a few small modifications for the
weighted elastic net case. Starting with selection probability, we can set \(s_j = 1\) and
replace \(\lambda_1\) with \(\lambda_1 u_j = \lambda_1(q_j-q_j^2)^\omega\) in
\Cref{eq:selection-probability}, which shows that \(\omega\) and \(\delta\) have
interchangeable effects for selection probability.

As far as expected value and variance of the weighted elastic net estimator is concerned,
the same expressions apply directly in the case of the weighted elastic net given \(s_j =
1\) for all \(j\) and replacing \(\lambda_1\) as in the previous paragraph and
\(\lambda_2\) with \(\lambda_2 (q_j - q_j^2)^\omega\). On the other hand, the asymptotic
results differ slightly as we now show.

\begin{theorem}
  \label{thm:weighted-elasticnet-bias-variance}
  Let \(\vec{x}_j\) be a binary feature with class balance \(q_j \in (0, 1)\) and take
  \(\lambda_1 > 0\), \(\lambda_2 > 0\), and \(\sigma_\varepsilon > 0\). For the
  weighted elastic net with weights \(u_j = v_j = (q_j-q_j^2)^\omega\) and \(\omega \geq 0\), it holds that
  \[
    \lim_{q_j \rightarrow 1^-} \E \hat{\beta}_j =
    \begin{cases}
      0                              & \text{if } 0 \leq \omega < 1, \\
      \frac{\beta^*n}{n + \lambda_2} & \text{if } \omega = 1,        \\
      \beta^*                        & \text{if } \omega > 1,
    \end{cases}
    \qquad
    \lim_{q_j \rightarrow 1^-} \var \hat{\beta}_j =
    \begin{cases}
      0      & \text{if } 0 \leq \omega < \frac{1}{2}, \\
      \infty & \text{if } \omega \geq \frac{1}{2}.
    \end{cases}
  \]
\end{theorem}

This result for expected value is similar to the one for the unweighted but normalized
elastic net. The only difference arises in the case when \(\omega = 1\), in which case the
limit is unaffected by \(\lambda_1\) in the case of the weighted elastic net. For variance,
the result mimics the result for the elastic net with normalization. The results for bias,
variance, and mean-squared error for the weighted elastic net are similar to those in
\Cref{fig:bias-var-onedim-lasso} and are plotted in
\Cref{fig:binary-onedim-bias-var-elnet}~(\Cref{sec:additional-results-biasvar}).

This wraps up our theoretical contributions in our paper. In the coming section, we will
turn to empirical experiments and demonstrate that our theoretical results both hold in
practice and, at least empirically, extend beyond our current assumptions.

\section{Experiments}
\label{sec:experiments}

In the following sections we present the results of our experiments. For all simulated data
we generate our response vector according to \(\vec{y} = \mat{X}\vec{\beta}^* +
\vec{\varepsilon},\) with \(\vec{\varepsilon} \sim \normal(\vec{0}, \sigma_\varepsilon^2
\mat{I})\). We consider two types of features: binary (quasi-Bernoulli) and quasi-normal
features. To generate binary vectors, we sample \(\ceil{nq_j}\) indexes uniformly at random
without replacement from \([n]\) and set the corresponding elements to one and the
remaining ones to zero. To generate quasi-normal features, we generate a linear sequence
\(\vec{w}\) with \(n\) values from \(10^{-4}\) to \(1 - 10^{-4}\), set \(x_{ij} =
\cdf^{-1}(w_i)\), and then shuffle the elements of \(\vec{x}_j\) uniformly at random.

We use a coordinate solver from \pkg{Lasso.jl}~\citep{kornblith2024} to optimize our
models, which we have based on the algorithm outlined by \citet{friedman2010}. All
experiments were coded using the Julia programming language~\citep{bezanson2017} and the
code is available
in the supplementary material and will be published online upon acceptance.
All simulated experiments were run for at least 100 iterations and, unless stated
otherwise, are presented as means $\pm$ one standard deviation (using bars or ribbons).

\subsection{Normalization in Lasso and Ridge Regression}%
\label{sec:experiments-lassoridge}

In this section we consider fitting the lasso and ridge regression to normalized datasets.
To normalize the data, we standardize all quasi-normal features. For binary features, we
center by mean and scale by \(s_j \propto (q_j-q_j^2)^\delta\).

\subsubsection{Variability and Bias in Estimates}

In our first experiment, we consider fitting the lasso to a simulated dataset with
\(n=500\) observations and \(p = \num{1000}\) features. The first 20 features correspond to
signals, with \(\beta_j^* = 1\), and otherwise we set \(\beta_j^*\) to 0. Furthermore, we
set the class balance of the first 20 features so that it increases geometrically from 0.5
to 0.99. For all other features we pick \(q_j\) uniformly at random in \([0.5,0.99]\). We
estimate the regression coefficients using the lasso, setting \(\lambda_1 = 2
\sigma_\varepsilon \sqrt{2 \log p }\), with \(\sigma_\varepsilon\) set to achieve a
signal-to-noise ratio (SNR) of 2. In addition, we introduce correlation between the
features by copying the first \(\ceil{\rho n/2}\) values from the first feature to each of
the following features.

The results~(\Cref{fig:binary-decreasing}) show that class balance has considerable effect,
particularly in the case of no scaling (\(\delta = 0\)), which corroborates our results in
\Cref{sec:theory-binary-features}. At \(q_j=0.99\), for instance, the estimate
(\(\hat{\beta}_{20}\)) is consistently zero when \(\delta = 0\). For larger values of
\(\delta\), we see that class imbalance leads to increased estimation variance in
accordance with our theory for the orthogonal case. The effect of correlation appears to
have no effect on this class-balance bias.

\begin{figure}[htpb]
  \centering
  \includegraphics[]{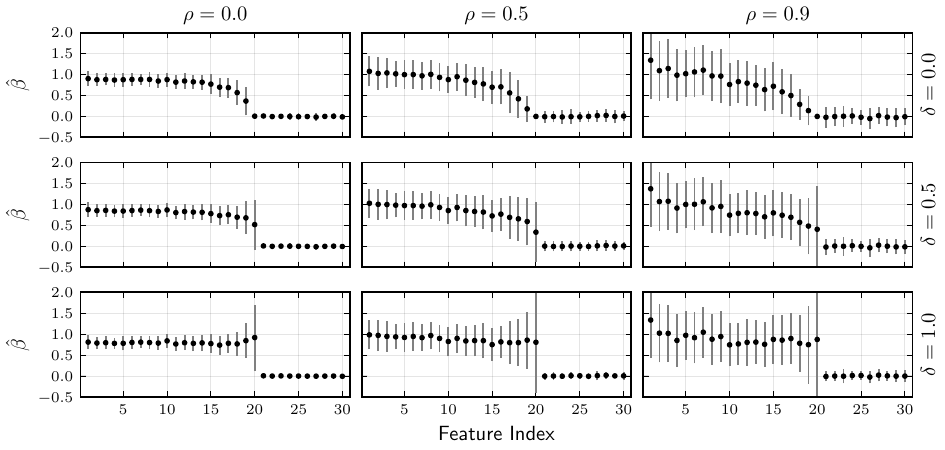}
  \caption{%
    Regression coefficients for a lasso problem with binary data with \(n = 500\) and \(p =
    \num{1000}\). For \(j \in \{1,2,\dots,p\}\), we set \(\beta_j^* = 1\) and
    let \(q_j\) increase geometrically from 0.5 to 0.99. For the remaining features,
    we pick \(q_j\) uniformly at random from \([0.5, 0.99]\) and
    set \(\beta_j^* = 0\). We show only the first 30 coefficients.
  }
  \label{fig:binary-decreasing}
\end{figure}

\subsubsection{Predictive Performance}
\label{sec:experiments-predictive-performance}

In this section we examine the influence of normalization on predictive performance for
three different datasets: \data{rhee2006}~\citep{rhee2006},
\data{eunite2001}~\citep{chen2004}, and
\data{triazines}~\citep{hirst1994,king1995}.\footnote{See \Cref{sec:data-summary} for
  details about these datasets.} We present the results for lasso and ridge regression in
\Cref{fig:hyperopt-contours}, which shows contour plots of the validation set error in
terms of normalized mean-squared error~(NMSE). We see that optimal setting of \(\delta\)
differs between the different datasets: for \data{eunite2001}, both the lasso and ridge are
quite insensitive to the type of normalization, and low error is attainable for the full
range of \(\delta\). For \data{rhee2006}, this holds for the lasso too, but not for ridge
regression, where a value in approximately \([0,3]\) is optimal. Finally, for
\data{triazines}, the problem is quite sensitive to the choice of \(\delta\) as well as the
type of model used.

\begin{figure}[htpb]
  \centering
  \includegraphics[]{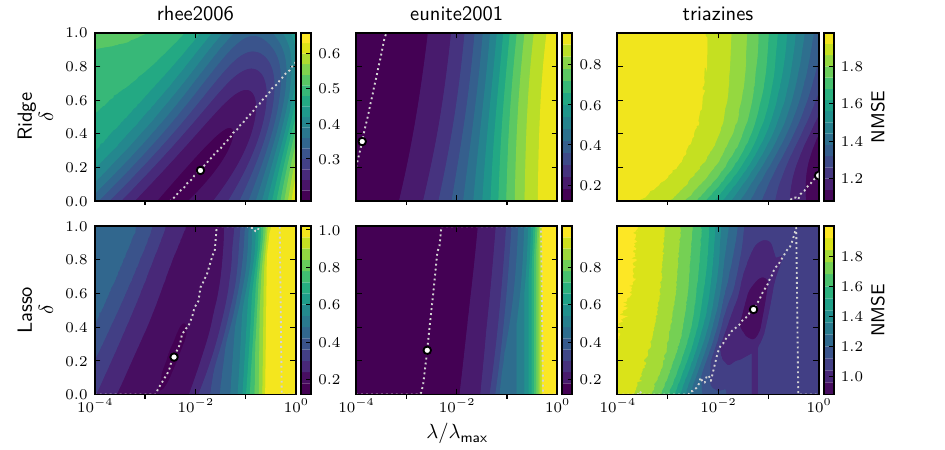}
  \caption{%
    Contour plots of normalized validation set mean-squared error (NMSE) for
    \(\delta\) and \(\lambda\) in lasso and ridge regression on three real data
    sets: \data{rhee2006}, \data{eunite2001}, and \data{triazines}. The dotted
    path shows the smallest NMSE as a function of \(\lambda\) and the circles
    mark combinations with the lowest error.
  }
  \label{fig:hyperopt-contours}
\end{figure}

Also see \Cref{sec:predictive-performance-support}, where we show how the support size of
the lasso solutions vary with \(\delta\) and \(\lambda\) and
\Cref{sec:predictive-performance-simulated}, where we complement these results with
experiments on simulated data under various class balances and signal-to-noise ratios,
again showing that normalization has a strong impact on predictive performance.

\subsubsection{Mixed Data}\label{sec:experiments-mixed-data}

In \Cref{sec:mixed-data} we discussed the issue of normalizing mixed data. Here, we examine
this issue empirically. We construct a quasi-normal feature with mean zero and standard
deviation 1/2 and a binary feature with varying class balance \(q_j\). We set the
signal-to-noise ratio to 0.5 and use \(n = \num{1000}\). These features are constructed so
that their effects are comparable under the notion of comparability that we introduced in
\Cref{sec:mixed-data} using \(\kappa = 2\). In order to preserve the comparability for the
baseline case when we have perfect class balance, we scale by \(s_j = 2 \times
(1/4)^{1-\delta}(q_j-q_j^2)^\delta\). Finally, we set \(\lambda\) to
\(\lambda_\text{max}/2\) and \(2\lambda_\text{max}\) for lasso and ridge regression
respectively.

The results~(\Cref{fig:lasso-ridge-comparison}) reflect our theoretical results from
\Cref{sec:theory}. In the case of the lasso, we need \(\delta =1\) (variance scaling) to
avoid the effect of class imbalance, whereas for ridge we instead need \(\delta =1/2\)
(standardization). As our theory suggests, this extra scaling mitigates this class-balance
bias at the cost of added variance.

\begin{figure}[htpb]
  \centering
  \includegraphics{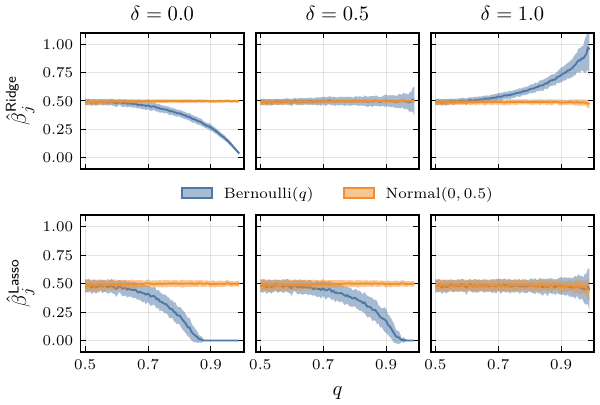}
  \caption{%
    Lasso and ridge estimates for a two-dimensional problem where one feature is a binary
    feature with class balance \(q_j\) (\(\bernoulli(q_j)\)) and the other is quasi-normal
    with standard deviation 1/2, (\(\normal(0, 0.5)\)).
  }
  \label{fig:lasso-ridge-comparison}
\end{figure}

\subsubsection{Interactions}\label{sec:experiments-interactions}

Next, we study the effects of normalization and class balance on interactions in the lasso.
Our example consists of a two-feature problem with an added interaction term given by
\(x_{i3} = x_{i1}x_{i2}\). The first feature is binary with class balance \(q\) and the
second quasi-normal with standard deviation 0.5. We use \(n=1000\), \(\lambda_1 = n/4\),
and normalize the binary feature by mean-centering and scaling by \(\kappa (q - q^2)\),
using \(\kappa = 2\). We consider two different strategies for choosing \(s_3\): in the
first strategy, which we call \emph{Strategy 1}, we simply standardize the resulting
interaction feature. This is a common strategy used, for instance in
\citet{bien2013,lim2015}. In the second strategy, \emph{Strategy 2}, we center with mean
and scale with \(s_1s_2\) (the product of the scales of the binary and normal features).

The results~(\Cref{fig:interactions}) show that only Strategy 2 estimates the effect of the
interaction correctly. Strategy 1, meanwhile, only selects the correct model if the class
balance of the binary feature is close to 1/2 and in general shrinks the coefficient too
much.

\begin{figure}[htpb]
  \centering
  \includegraphics[]{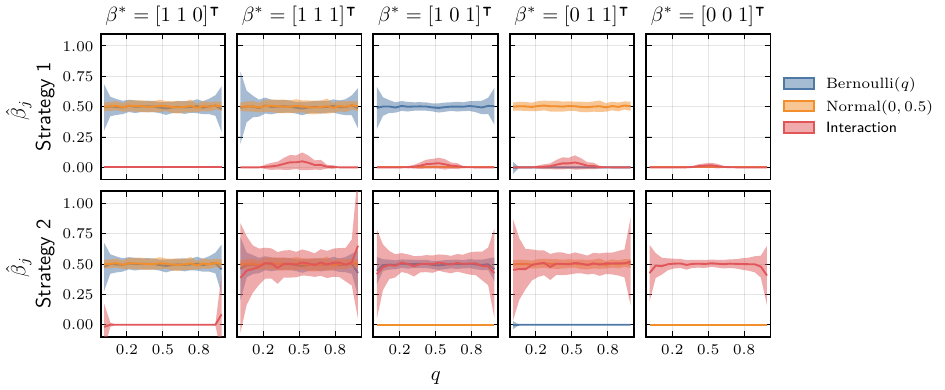}
  \caption{%
    Lasso estimates for a problem with a binary feature, a quasi-normal
    feature, and an interaction feature. We have varied the true signal
    \(\bm{\beta}^*\) and use two different normalization strategies for the
    interaction, where
    Strategy 1 represents standardization and Strategy 2 is mean-centering
    together with scaling by \(s_1 s_2\).
  }
  \label{fig:interactions}
\end{figure}

\subsection{The Weighted Elastic Net}

The weighted elastic net can be used as an alternative to normalization to correct for
class balance bias when \(\lambda_1 > 0\) and \(\lambda_2 >0\). To simplify the
presentation, we parameterize the elastic net as \(\lambda_1 = \alpha \lambda \) and
\(\lambda_2 = (1-\alpha) \lambda\), so that \(\alpha\) controls the balance between the
ridge and lasso. We conduct an experiment with the same setup as in
\Cref{sec:experiments-mixed-data}, but here we use the weighted elastic net instead with
\(\alpha = 0.5\). We use \(n=1000\) and vary \(\omega\), using the weights \(u_j = v_j =
(q_j - q_j^2)^{\omega}\) as we suggested in \Cref{sec:binary-weighting}. Our results
(\Cref{fig:mixed-data-elnet}) show that \(\omega = 1\) leads to seemingly unbiased
estimates.

\begin{figure}[htpb]
  \centering
  \includegraphics{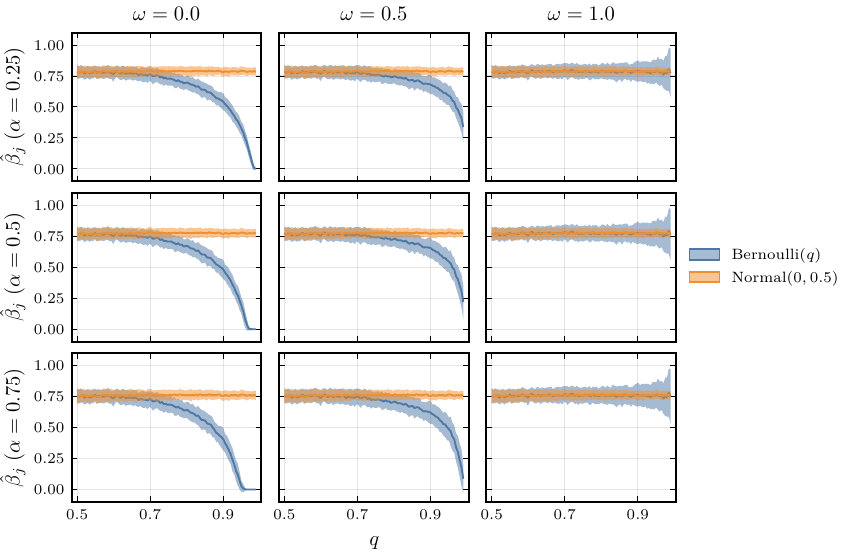}
  \caption{%
    Weighted elastic net estimates for \(\alpha = 0.5\) for a problem with a binary
    feature with class balance \(q\) (\(\bernoulli(q)\)) and quasi-normal
    with standard deviation 1/2 (\(\normal(0, 0.5)\)). \(\omega\) indicates
    the scaling of the penalty weights.
  }
  \label{fig:mixed-data-elnet}
\end{figure}

\subsection{Additional Results}

In \Cref{sec:additional-experiments} we present extended results of our experiments as well
as several additional experiments. For instance, we also consider an experiment on the
Boston housing dataset in \Cref{sec:dichotomization}, where we try to estimate the effect
of normalization on dichotomized versions of the features and find that \(\delta = 1\)
leads to feature ranks that best correspond to the linear regression solution.

\section{Discussion}


This is the first paper to study the effects of normalization in lasso, ridge, and elastic
net regression with binary data. We have discovered that the class balance (proportion of
ones) of these binary features has a pronounced effect on both lasso and ridge estimates
and that this effect depends on the type of normalization used. For the lasso, for
instance, features with large class imbalances stand little chance of being selected if the
binary features are standardized or, to an even greater extent, if they are not scaled at
all---even if their relationships with the response are strong. Not scaling binary features
is a common approach in practice, for instance recommended in the \pkg{scikit-learn}
documentation when the data is sparse~\citep{scikit-learndevelopers2025}. As is
standardization, which is the default in many software packages for the lasso, such as
\pkg{glmnet}~\citep{friedman2010}, as well as the practice taken by countless applications
in research. All in all, this means that the class-balance bias effect is likely pervasive
in practice and risks having already influenced the conclusions drawn in analyses of many
datasets.

The driver of this bias is the relationship between the variance of the feature and type of
normalization. This works as expected for normally distributed features. But for binary
features it means that a one-unit change is treated differently depending on the
corresponding feature's class balance, which we believe may surprise some. We have,
however, shown that scaling binary features with standard deviation in the case of ridge
regression and variance in the case of the lasso mitigates this effect, but that doing so
comes at the price of increased variance. This effectively means that the choice of
normalization constitutes a bias--variance trade-off.

We have also studied the case of mixed data: designs that include both binary and normally
distributed features~(\Cref{sec:mixed-data}). In this setting, our first finding is that
there is an implicit relationship between the choice of normalization and the manner in
which regularization affects binary vis-à-vis normally distributed features, even when the
binary feature is perfectly balanced. The choice of max--abs normalization, for instance,
leads to a specific weighting of the effects of binary features relative to those of normal
features.

For interactions between binary and normal features~(\Cref{sec:experiments-interactions}),
our conclusion is that the interaction feature---contrary to what recent literature on
interactions in the lasso recommends---should be scaled with the \emph{product} of the
standard deviation of the normal feature and variance of the binary feature to avoid this
effect of class imbalance. We have not seen this recommendation in the literature before,
but it is a natural extension of our other results.

We note that our theoretical results are limited by a few assumptions: 1) a fixed feature
matrix \(\bm{X}\), 2) normal and independent errors, and 3) orthogonal features. The first
and second of these assumptions are standard in the literature. The third assumption on
orthogonality, however, is strong and rarely satisfied in practice. Yet, as we show in
\Cref{sec:orthogonality-assumption} and our experiments, the assumption does not in fact
appear to be restrictive for our results, which, at least empirically, hold under much more
general settings. We have also focused on the case of binary and continuous features here,
but are convinced that categorical features are also of interest and might raise additional
challenges with respect to normalization. Finally, most of our results are restricted to
least-squares loss, but since all generalized linear models (GLMs) are parameterized by the
linear predictor, which we have shown to depend directly on class balance, we believe that
our results are also relevant for other loss functions. Our initial results in
\Cref{sec:normalization-tuning} seem to support this claim, but we defer further
investigation of this to future work.

Regularized regression models are widely used in practice, and are staples of popular
machine learning and statistical software packages such as
\pkg{glmnet}~\citep{friedman2010}, \pkg{scikit-learn}~\citep{pedregosa2011},
\pkg{mlpack}~\citep{curtin2023}, \pkg{skglm}~\citep{bertrand2022},
\pkg{LIBLINEAR}~\citep{fan2008a}, and \pkg{MATLAB}~\citep{themathworksinc.2022}. Our
results suggest that the choice of normalization is an important aspect of using these
models that, in spite of the popularity of these methods, has so far been overlooked. We
hope that our results will motivate researchers and practitioners to consider the choice of
normalization more carefully in the future.

\bibliography{main}
\bibliographystyle{tmlr}

\appendix

\section{Orthogonality Assumption}
\label{sec:orthogonality-assumption}

A key assumption in our theoretical results is that the features are orthogonal to one
another. In general, this is a strong assumption that is rarely satisfied in practice. In
the high-dimensional setting with binary data, it is even impossible to achieve. In spite
of this, it turns out that the assumption holds little bearing on our results.

The primary reason for this is a well-known behavior of regularized estimators when
features are correlated. Since information about the effect is shared between the
correlated features, the objective can attain a lower value by favoring one of the features
over the others. The effect is particularly strong in the lasso, which is known to select
only one of the correlated features and ignore the others given a sufficiently large
correlation and penalty strength.

A second reason for why the assumption is not as restrictive as it may seem is that the
correlation between two features, at least one being binary, tends to zero as class balance
increases towards one. We first demonstrate this empirically in the following experiment.
We consider a setting with two binary features: the first, \(\bm{x}_1\) with class balance
\(q_1 = 0.5\), and the second, \(\bm{x}_2\), with varying balance \(q_2 \in [0.5, 0.9]\).
We set the true effect to \(\beta_1^* = \beta_2^* = 1\) and the level of correlation,
\(\rho\) to three different levels (0, 0.4, and 0.6). The noise level
\(\sigma_\varepsilon\) is set to obtain a signal-to-noise ratio of 1. We generate
\(n=\num{10000}\) observations and fit the lasso with \(\lambda = \lambda_\text{max}/2\).

The results are shown in \Cref{fig:orthogonality} where we see that the effect of
correlation has no impact on the shrinkage imposed from decreasing class balance of the
second feature. The results in fact suggest that the effect of \(q_2\) is \emph{stronger}
when the features are correlated, which is due to the nature of the lasso that we
previously discussed.

\begin{figure}[htpb]
  \centering
  \includegraphics[]{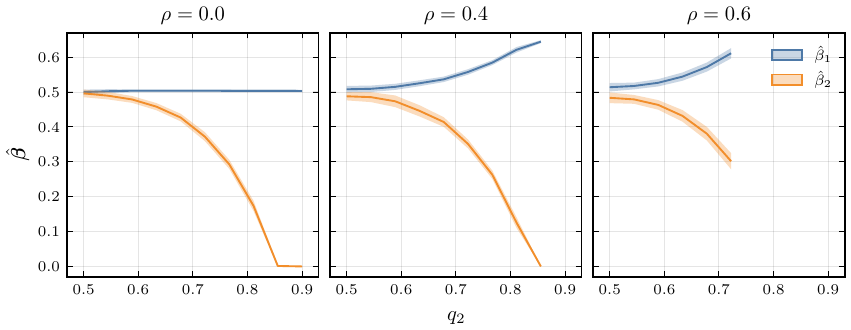}
  \caption{%
    Estimates of the regression coefficients from the lasso,
    \(\hat{\vec{\beta}}\), for the two features in the experiment in
    \Cref{sec:orthogonality-assumption}. The first feature is binary with class
    balance \(q_1 = 0.5\) and the second is binary with class balance \(q_2 \in [0.5, 0.9]\).
    The features are correlated with correlation \(\rho\). The plot shows means and
    95\% normal confidence intervals averaged over 100 iterations. Note
    that it's impossible to achieve high levels of correlation when the second
    feature is highly imbalanced, which is why the results for \(\rho = 0.4\) and \(0.6\)
    do not cover the whole range of \(q_2\) values.
  }
  \label{fig:orthogonality}
\end{figure}

In the following theorems and corollary, we provide further evidence of this behavior by
showing that the correlation between a binary feature and a continuous or binary feature
tends to zero as the class balance increases.

In conclusion, although the orthogonality assumption is strong in general, unrealistic, and
in many other areas of statistics make results hard to generalize to a wider setting, it
does not seem to have a significant impact on our results.

\begin{theorem}[Correlation for Dichotomized Normal Variables]
  Let \(X\) and \(Y\) be two standard normal random variables with correlation \(\rho\).
  Define \(Z = \bm{1}[Y > \alpha]\) where \(\alpha = \cdf^{-1}(q)\) is the
  quantile at which \(Y\) is dichotomized.
  Then
  \[
    \cor(X,Z)
    = \frac{\rho\pdf(\alpha)}{\sqrt{q \bigl(1 - q\bigr)}} \rightarrow 0 \quad\text{as}\quad q \rightarrow 1
  \]
\end{theorem}
\begin{proof}
  Let \(Z = \mathbf{1}[Y > \alpha]\) with \(\alpha = \cdf^{-1}(q)\). Then,
  using the law of total expectation, we have
  \[
    \E(XZ) = \E\bigl(\E(XZ \mid Y)\bigr) = \E\bigl(\rho Y \bm{1}[Y > \alpha]\bigr)
    = \rho \int_{\alpha}^{\infty} y \pdf(y)\,\du y = \rho \pdf(\alpha).
  \]
  Since \(\var(X) = 1\) and \(\var(Z) = q (1 - q)\), it follows that
  \[
    \cor(X,Z) = \frac{\E(XZ)}{\sqrt{\var(X)\,\var(Z)}} = \frac{\rho \pdf(\alpha)}{\sqrt{q \bigl(1 - q\bigr)}}.
  \]
  And since \(\pdf(\alpha) \to 0\) exponentially fast as \(q \to 1\), it follows that
  \(\cor(X,Z) \to 0\) as \(q \to 1\).
\end{proof}

\begin{theorem}[Correlation with Bernoulli Variable]
  Let \(X\) be a continuous random variable and \(Y\) be a Bernoulli random variable defined as:
  \[
    Y = \begin{cases}
      1 & \text{with probability } p,   \\
      0 & \text{with probability } 1-p.
    \end{cases}
  \]
  Let \(\mu_1 = \E \left( X \mid Y=1 \right)\), \(\mu_0 = \E \left( X \mid Y=0 \right)\) and
  $\sigma_X^2 = V \left(X \right)$ Then the correlation between \(X\) and \(Y\) is given by:
  \[
    \rho_{X,Y} = \frac{(\mu_1 - \mu_0)}{\sigma_X}\sqrt{p(1-p)}.
  \]
\end{theorem}

\begin{proof}
  The correlation \(\rho_{X,Y}\) is defined as:
  \[
    \rho_{X,Y} = \frac{\mathrm{Cov}(X,Y)}{\sqrt{\mathrm{Var}(X)\mathrm{Var}(Y)}}.
  \]
  We have \(\mathrm{Var}(Y)=p(1-p)\). Using the law of total covariance:
  \begin{align*}
    \mathrm{Cov}(X,Y) & = \E (XY) - \E(X) \E(Y)           \\
                      & = p\mu_1 - (p\mu_1 + (1-p)\mu_0)p \\
                      & = p(1-p)(\mu_1 - \mu_0).
  \end{align*}
  Since $\Var (Y)=p(1-p)$ the result follows.
\end{proof}

\begin{corollary}[Gaussian-Bernoulli Case]
  Suppose \(X\) and \(Z\) are jointly Gaussian random variables with:
  \[
    X, Z \sim N(0,1), \quad \mathrm{Corr}(X,Z)=\rho,
  \]
  and define \(Y=\mathbf{1}[Z>\alpha]\). Then the correlation between \(X\) and \(Y\) is:
  \[
    \rho_{X,Y} = \frac{\rho\,\phi(\alpha)}{\sqrt{\Phi(\alpha)(1 - \Phi(\alpha))}},
  \]
  where \(\phi(\alpha)\) and \(\Phi(\alpha)\) are the PDF and CDF of the standard normal
  distribution, respectively. Further for \(q = \Phi(\alpha)\), we have: $$ \rho_{X,Y} \to 0
    \quad \text{as} \quad q \to 1. $$
\end{corollary}

\begin{proof}
  From the theorem above, we identify:
  \begin{itemize}
    \item \(p = P(Z>\alpha) = 1-\Phi(\alpha)\).
    \item Since \((X,Z)\) is jointly normal, \(X|Z=z \sim N(\rho z, 1-\rho^2)\), we have:
          \[
            \mu_1 = E[X \mid Z>\alpha] = \rho E[Z \mid Z>\alpha] = \rho \frac{\phi(\alpha)}{1-\Phi(\alpha)}.
          \]
    \item Similarly,
          \[
            \mu_0 = E[X \mid Z \leq \alpha] = \rho E[Z \mid Z\leq\alpha] = -\rho\frac{\phi(\alpha)}{\Phi(\alpha)}.
          \]
  \end{itemize}
  Thus:
  \[
    \mu_1 - \mu_0 = \rho\left(\frac{\phi(\alpha)}{1-\Phi(\alpha)} + \frac{\phi(\alpha)}{\Phi(\alpha)}\right) = \frac{\rho\phi(\alpha)}{\Phi(\alpha)(1-\Phi(\alpha))}.
  \]
  Substituting these results into the theorem, we obtain:
  \begin{align*}
    \rho_{X,Y} & = \frac{\rho\phi(\alpha)}{\Phi(\alpha)(1-\Phi(\alpha))}\sqrt{\Phi(\alpha)(1-\Phi(\alpha))} \\[6pt]
               & = \frac{\rho\,\phi(\alpha)}{\sqrt{\Phi(\alpha)(1 - \Phi(\alpha))}}.
  \end{align*}
  Finally, note that $\rho(\alpha)$ is bounded hence $\rho_{X,Y} \rightarrow 0 $ as $q=\Phi(\alpha) \rightarrow 0$.
\end{proof}

One can also compute bounds for two correlated Bernoulli variables: Let \(X \sim
\mathrm{Bernoulli}(p)\) and \(Y \sim \mathrm{Bernoulli}(q)\) with \(0 < p,q < 1\). Their
correlation coefficient is given by
\[
  \rho = \frac{P(X=1,Y=1)-pq}{\sqrt{p(1-p)q(1-q)}},
\]
where the joint probability \(P(X=1,Y=1)\) satisfies the Fréchet bounds:
\[
  \max\{0,\, p+q-1\} \le P(X=1,Y=1) \le \min\{p,\, q\}.
\]
Thus, the extreme values of \(\rho\) are:
\[
  \rho_{\max}(p,q) = \frac{\min\{p,q\}-pq}{\sqrt{p(1-p)q(1-q)}}
\]
and
\[
  \rho_{\min}(p,q) = \frac{\max\{0,\, p+q-1\}-pq}{\sqrt{p(1-p)q(1-q)}}.
\]

\section{Additional Theory}

\subsection{Maximum--Absolute and Min--Max Normalization for Normally Distributed Data}%
\label{sec:maxabs-theory}

In \Cref{thm:maxabs-gev}, we show that the scaling factor in the max--abs method converges
in distribution to a Gumbel distribution.

\begin{theorem}
  \label{thm:maxabs-gev}
  Let \(X_1, X_2, \dots, X_n\) be a sample of normally distributed random variables, each with mean \(\mu\) and standard deviation \(\sigma\). Then
  \[
    \lim_{n \rightarrow \infty}\Pr\left(\max_{i \in [n]} |X_i| \leq x\right) = G(x),
  \]
  where \(G\) is the cumulative distribution function of a Gumbel distribution with
  parameters
  \[
    b_n = F_Y^{-1}(1 - 1/n)\quad \text{and} \quad a_n = \frac{1}{n f_Y(\mu_n)},
  \]
  where \(f_Y\) and \(F_Y^{-1}\) are the probability distribution function and quantile
  function, respectively, of a folded normal distribution with mean \(\mu\) and standard
  deviation \(\sigma\).
\end{theorem}

The gist of \Cref{thm:maxabs-gev} is that the limiting distribution of \(\max_{i \in
  [n]}|X_i|\) has expected value \(b_n + \gamma a_n\), where \(\gamma\) is the
Euler-Mascheroni constant, which shows that the scaling factor depends on the sample size.
In \Cref{fig:maxabs-gev}, we observe empirically that the limiting distribution agrees well
with the empirical distribution in expected value even for small values of \(n\).

In \Cref{fig:maxabs-n} we show the effect of increasing the number of observations, \(n\),
in a two-feature lasso model with max-abs normalization applied to both features. The
coefficient corresponding to the Normally distributed feature shrinks as the number of
observation \(n\) increases. Since the expected value of the Gumbel distribution diverges
with \(n\), this means that there's always a large enough \(n\) to force the coefficient in
a lasso problem to zero with high probability.

\begin{figure}[htpb]
  \centering
  \subfigure[Theoretical versus empirical distribution of the maximum absolute value of normally distributed random variables.\label{fig:maxabs-gev}]{\includegraphics[]{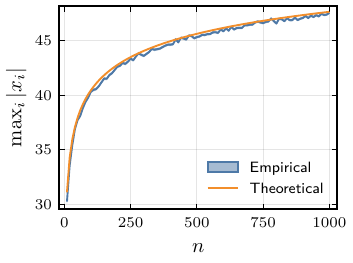}}%
  \hspace{1cm}
  \subfigure[Estimation of mixed features under maximum absolute value scaling\label{fig:maxabs-n}]{\includegraphics[]{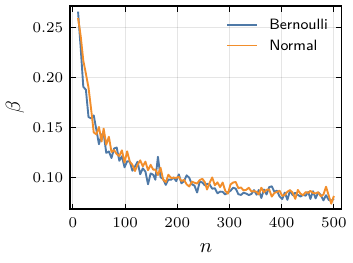}}
  \caption{%
    Effects of maximum absolute value scaling
  }
\end{figure}

For min--max normalization, the situation is similar and we omit the details here. The main
point is that the scaling factor is strongly dependent on the sample size, which makes it
unsuitable for normally distributed data in several situations, such as on-line learning
(where sample size changes over time) or model validation with uneven data splits.

\subsection{Solution to the Elastic Net}%
\label{sec:elastic-net-estimator}

Let \((\hat{\beta}_0^{(n)}, \hat{\vec{\beta}}^{(n)})\) be a solution to the problem in
\Cref{eq:elastic-net}. Expanding the function, we have
\[
  \frac{1}{2}\left( \vec y^\T \vec y - 2(\tilde{\mat{X}}\vec{\beta} + \beta_0)^\T\vec{y} + (\tilde{\mat{X}}\vec{\beta} + \beta_0)^\T(\tilde{\mat{X}}\vec{\beta} + \beta_0)\right)
  + \lambda_1 \lVert \vec\beta \rVert_1 + \frac{\lambda_2}{2}\lVert \vec \beta \rVert_2^2.
\]
Taking the subdifferential with respect to \(\vec{\beta}\) and \(\beta_0\), the KKT
stationarity condition yields the following system of equations:
\begin{equation}
  \label{eq:kkt-elasticnet}
  \begin{cases}
    \tilde{\mat{X}}^\T(\tilde{\mat{X}}\vec{\beta} + \beta_0 - \vec{y}) + \lambda_1 g + \lambda_2 \vec\beta \ni \vec{0}, \\
    n \beta_0 + (\tilde{\mat{X}}\vec{\beta})^\T \vec{1} - \vec{y}^\T \vec{1} = 0,
  \end{cases}
\end{equation}
where \(g\) is a subgradient of the \(\ell_1\) norm that has elements \(g_i\) such that
\[
  g_i \in
  \begin{cases}
    \{\sign{\beta_i}\} & \text{if } \beta_i \neq 0, \\
    [-1, 1]            & \text{otherwise}.
  \end{cases}
\]

\subsubsection{Orthogonal Features}

If the features of the normalized design matrix are orthogonal, that is,
\(\tilde{\mat{X}}^\intercal \tilde{\mat{X}} = \diag\left(\tilde{\vec{x}}_1^\T
\tilde{\vec{x}}_1, \dots, \tilde{\vec{x}}_p^\intercal \tilde{\vec{x}}_p\right) \), then
\Cref{eq:kkt-elasticnet} can be decomposed into a set of \(p + 1\) conditions:
\[
  \begin{cases}
    \tilde{\vec{x}}_j^\T (\tilde{\vec{x}}_j \beta_j + \ones \beta_0 - \vec{y}) + \lambda_2 \beta_j + \lambda_1 g \ni 0, & j \in [p], \\
    n \beta_0 + (\tilde{\mat{X}}\vec{\beta})^\T \vec{1} -  \vec{y}^\T \ones = 0.
  \end{cases}
\]
The inclusion of the intercept ensures that the locations (means) of the features do not
affect the solution (except for the intercept itself). We will therefore from now on assume
that the features are mean-centered so that \(c_j = \bar{\vec{x}}_j\) for all \(j\) and
therefore \(\tilde{\vec{x}}_j^\T \ones = 0\). A solution to the system of equations is then
given by the following set of equations~\citep{donoho1994}:
\begin{equation*}
  \hat{\beta}^{(n)}_j = \frac{\st_{\lambda_1}\left(\tilde{\vec{x}}_j^\T \vec{y}\right)}{\tilde{\vec{x}}_j^\T \tilde{\vec{x}}_j + \lambda_2},
  \qquad
  \hat{\beta}_0^{(n)} = \frac{\vec{y}^\T \ones}{n},
\end{equation*}
where \(\st_\lambda(z) = \sign(z) \max(|z| - \lambda, 0)\) is the soft-thresholding
operator.

\subsection{Bias and Variance of the Elastic Net Estimator}
\label{sec:bias-var-deriv}

Here, we derive the results in \Cref{sec:theory} in more detail. Let
\[
  {Z_j} = \tilde{\vec{x}}_j^\T \vec{y} = \tilde{\vec{x}}_j^\T(\mat{X}\vec{\beta}^* + \vec{\varepsilon}) = \tilde{\vec{x}}_j^\T (\vec{x}_j\beta_j^* + \bm{\varepsilon})
  \qquad
  \text{and}
  \qquad
  d_j = s_j(\tilde{\vec{x}}_j^\T \tilde{\vec{x}}_j + \lambda_2)
\]
so that \(\hat{\beta}_j = \st_{\lambda_1}({Z_j})/d_j\). Since \(d_j\) is fixed under our
assumptions, we focus on \(S_{\lambda_1}({Z_j})\). First observe that since \(c_j =
\bar{\bm{x}}_j\),
\[
  \begin{aligned}
    \tilde{\vec{x}}_j^\T \tilde{\vec{x}}_j & = \frac{1}{s_j^2}(\vec{x}_j - c_j)^\T (\vec{x}_j - c_j) = \frac{\vec{x}_j^\T\vec{x}_j - nc_j^2}{s^2_j} = \frac{n \nu_j}{s_j^2}, \\
    \tilde{\vec{x}}_j^\T \vec{x}_j         & = \frac{1}{s_j}(\vec{x}_j^\T \vec{x}_j - \vec{x}_j^\T \ones c_j) = \frac{n \nu_j}{s_j},
  \end{aligned}
\]
where \(\nu_j\) is the uncorrected sample variance of \(\vec{x}_j\). This means that
\begin{equation}
  {Z_j} = \tilde{\vec{x}}_j^\T (\vec{x}_j\beta_j^* + \bm{\varepsilon}) = \frac{\beta_j^* n \nu_j- \vec{x}_j^\T \vec{\varepsilon} + c_j \ones^\intercal \bm{\varepsilon}}{s_j}
  \qquad\text{and}\qquad
  d_j = s_j\left(\frac{n \nu_j}{s_j^2} + \lambda_2\right).
\end{equation}
For the expected value and variance of \({Z_j}\) we then have
\begin{align*}
  \E {Z_j}   & = \mu_j = \E \left( \tilde{\vec{x}}_j^\T (\vec{x}_j\beta^*_j + \vec{\varepsilon}) \right)  = \tilde{\vec{x}}_j^\T\vec{x}_j \beta^*_j = \frac{\beta_j^* n \nu_j}{s_j},            \\
  \var {Z_j} & = \sigma_j^2 = \var\left(\tilde{\vec{x}}_j ^\T \vec{\varepsilon}\right) = \tilde{\vec{x}}_j^\T \tilde{\vec{x}}_j\sigma_\varepsilon^2 = \frac{n\nu_j\sigma_\varepsilon^2}{s_j^2}.
\end{align*}

The expected value of the soft-thresholding estimator is
\begin{equation*}
  \E \st_\lambda({Z_j}) = \int_{-\infty}^\infty \st_\lambda(z) f_{Z_j}(z) \du z
  = \int_{-\infty}^{-\lambda}(z + \lambda)f_{Z_j}(z) \du z + \int_{\lambda}^\infty (z - \lambda)f_{Z_j}(z) \du z.
\end{equation*}
And then the bias of \(\hat\beta_j\) with respect to the true coefficient \(\beta_j^*\) is
\begin{equation*}
  \E \hat\beta_j - \beta_j^* = \frac{1}{d_j}\E \st_\lambda({Z_j}) - \beta^*_j.
\end{equation*}

Finally, we note that the variance of the soft-thresholding estimator is
\begin{equation}
  \label{eq:st-variance}
  \var {S_\lambda({Z_j})} = \int_{-\infty}^{-\lambda}(z + \lambda)^2f_{Z_j}(z) \du z + \int_{\lambda}^\infty (z - \lambda)^2 f_{Z_j}(z) \du z - \left(\E \st_\lambda({Z_j})\right)^2
\end{equation}
and that the variance of the elastic net estimator is therefore
\begin{equation*}
  \var \hat\beta_j = \frac{1}{d_j^2} \var \st_\lambda({Z_j}).
\end{equation*}

\subsubsection{Normally Distributed Noise}%
\label{sec:normally-distributed-noise}

We now assume that \(\vec{\varepsilon}\) is normally distributed. Then
\[
  {Z_j} \sim \normal\left(\mu_j = \tilde{\vec{x}}_j^\T\vec{x}_j \beta_j^*, \sigma_j^2 = \tilde{\vec{x}}_j^\T\tilde{\vec{x}}_j \sigma_\varepsilon^2 \right).
\]
Let \(\theta_j = -\mu_j -\lambda_1 \) and \(\gamma_j = \mu_j - \lambda_1\). Then the
expected value of soft-thresholding of \({Z_j}\) is
\begin{align*}
  \E \st_{\lambda_1}({Z_j}) & = \int_{-\infty}^\frac{\theta_j}{\sigma_j} (\sigma_j u - \theta_j) \pdf(u) \du u + \int_{-\frac{\gamma_j}{\sigma_j}}^\infty (\sigma_j u + \gamma_j) \pdf(u) \du u                                               \nonumber \\
                            & = -\theta_j \cdf\left(\frac{\theta_j}{\sigma_j}\right) - \sigma_j \pdf\left(\frac{\theta_j}{\sigma_j}\right) + \gamma_j \cdf\left(\frac{\gamma_j}{\sigma_j}\right) + \sigma_j \pdf\left(\frac{\gamma_j}{\sigma_j}\right)
\end{align*}
where \(\pdf(u)\) and \(\cdf(u)\) are the probability density and cumulative distribution
functions of the standard normal distribution, respectively. Computing \Cref{eq:st-variance} gives us
\begin{align*}
  \var{S_\lambda(Z_j)} & = \frac{\sigma_j^2}{2} \left( \erf\left(\frac{\theta_j}{\sigma_j\sqrt{2}}\right) \phantom{-{}} \frac{\theta_j}{\sigma_j}\sqrt{\frac{2}{\pi}} \exp\left(-\frac{\theta_j^2}{2\sigma_j^2}\right) + 1 \right)  \nonumber  \\
                       & \phantom{={}} + 2 \theta_j \sigma_j \pdf \left(\frac{\theta_j}{\sigma_j}\right) + \theta_j^2 \cdf\left(\frac{\theta_j}{\sigma_j}\right) \nonumber                                                                     \\
                       & \phantom{={}} + \frac{\sigma_j^2}{2} \left( \erf\left(\frac{\gamma_j}{\sigma_j\sqrt{2}}\right) - \frac{\gamma_j}{\sigma_j}\sqrt{\frac{2}{\pi}} \exp\left(-\frac{\gamma_j^2}{2\sigma_j^2}\right) + 1 \right) \nonumber \\
                       & \phantom{={}} + 2 \gamma_j \sigma_j \pdf \left(\frac{\gamma_j}{\sigma_j}\right) + \gamma_j^2 \cdf\left(\frac{\gamma_j}{\sigma_j}\right) \nonumber                                                                     \\
                       & \phantom{={}} - \big(\E \st_{\lambda_1}({Z_j})\big)^2.
\end{align*}

\subsection{Derivation of Estimate in the Noiseless Case}%
\label{sec:noiseless-estimator}

If we assume \(s_j = (q_j - q_j^2)^\delta\), then \Cref{eq:noiseless-estimator} becomes
\begin{equation*}
  \label{eq:noiseless-estimator-2}
  \hat{\beta}_j = \frac{\st_{\lambda_1}\left(\beta_j^* n (q_j - q_j^2)^{1 - \delta}\right)}{n(q_j - q_j^2)^{1 - \delta} + (q_j - q_j^2)^\delta \lambda_2  }.
\end{equation*}

If we are in the lasso case, then \(\lambda_2 = 0\) and
\[
  \hat{\beta}_j = \frac{\st_{\lambda_1}\left(\beta_j^* n (q_j - q_j^2)^{1 - \delta}\right)}{n(q_j - q_j^2)^{1 - \delta}},
\]
so we need to choose \(\delta = 1\), and hence \(s_j = q_j - q_j^2\), to get rid of the
dependency on \(q_j\). For the ridge case, we instead have \(\lambda_1 = 0\) and hence
\[
  \hat{\beta}_j
  = \frac{\beta_j^* n (q_j - q_j^2)^{1 - \delta}}{n(q_j - q_j^2)^{1 - \delta} + (q_j - q_j^2)^\delta \lambda_2  }
  = \frac{\beta_j^* n}{n + (q_j - q_j^2)^{2\delta - 1} \lambda_2},
\]
which shows that we need to choose \(\delta = 1/2\) and hence \(s_j = \sqrt{q_j - q_j^2}\),
to get rid of the dependency on \(q_j\) in this case.

Note that if \(\lambda_1,\lambda_2 > 0\) (the elastic net case), then there is no choice of
\(\delta\) that will make the estimator independent of \(q_j\).

\subsection{Bias and Variance for Ridge Regression}%
\label{sec:ridge-variance}

\begin{corollary}[Variance in Ridge Regression]
  \label{cor:ridge-variance}
  Assume the conditions of \Cref{thm:classbalance-bias} hold, except that
  \(\lambda_1 = 0\). Then
  \[
    \lim_{q_j \rightarrow 1^-} \var \hat{\beta}_j =
    \begin{cases}
      0                                          & \text{if } 0 \leq \delta < 1/4, \\
      \frac{\sigma_\varepsilon^2 n}{\lambda_2^2} & \text{if } \delta = 1/4,        \\
      \infty                                     & \text{if } \delta > 1/4.
    \end{cases}
  \]
\end{corollary}

\subsection{Centering and Interaction Features}%
\label{sec:centering-interactions}

The main motivation for centering is that it removes correlation between the main features
and the interaction, which would otherwise affect the estimates due to the regularization.
Centering normal features is also important because it ensures that their means do not
factor into the estimation of their effects, which is otherwise the case since the variance
of \(\bm{x}_3\) would then be \(q_1(\sigma^2 + \mu^2(1 - q_1))\) in the case when
\(\bm{x}_1\) is centered and \((q_1 - q_1^2)(\sigma^2 + \mu^2)\) otherwise. Centering
binary features is also important because the variance of the interaction term is otherwise
\(q_1\sigma^2\) (provided \(\bm{x}_2\) is centered), which would mean that the encoding of
values of the binary feature (e.g. \(\{0,1\}\) versus \(\{-1, 1\}\)) would affect the
interaction term.

\subsection{Extended Results on Bias and Variance for Ridge, Lasso, and Elastic Net Regression}%
\label{sec:additional-results-biasvar}

In \Cref{fig:binary-onedim-bias-var-elnet}, we show bias, variance, and mean-squared error
for the weighted elastic net. We see that the behavior of bias as \(q_j \rightarrow 1^-\)
depends on noise level and that there is a bias--variance trade-off with respect to
\(\omega\). As in \Cref{sec:mixed-data}, we modify the weighting factor to have
comparability under \(\kappa = 2\).

\begin{figure}[htpb]
  \centering
  \includegraphics[]{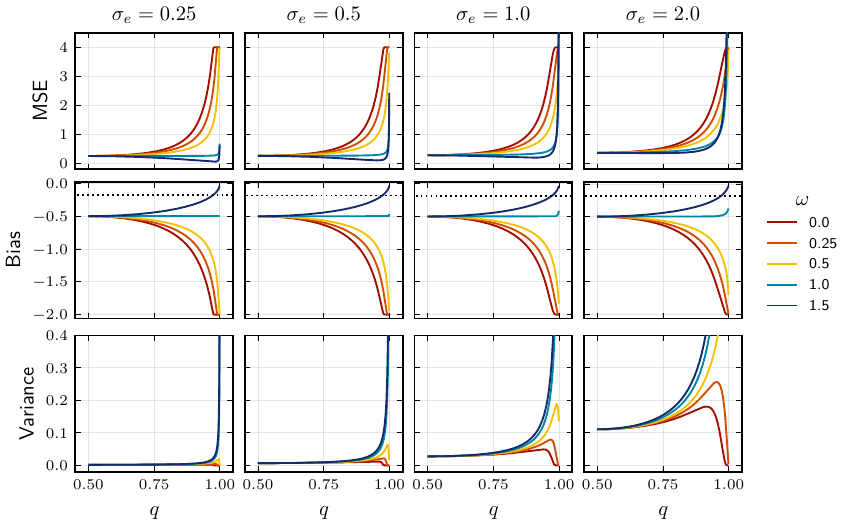}
  \caption{%
    Bias, variance, and mean-squared error in the case of the one-dimensional weighted elastic
    net. The measures are shown for different noise levels (\(\sigma_\varepsilon\)), class
    balances (\(q_j\)), and values of (\(\omega\)), which controls the weights that are set to
    \(u_j = v_j = 2\times 4^{\omega - 1}(q-q^2)^\omega\) in order for the results to be
    comparable across different values of \(\omega\). The dotted lines represent the asymptotic
    bias of the estimator in the case of \(\omega = 1\). In the case of \(\omega > 1\), the
    limit of the bias is zero.
  }
  \label{fig:binary-onedim-bias-var-elnet}
\end{figure}

\begin{figure}[htpb]
  \centering
  \includegraphics[]{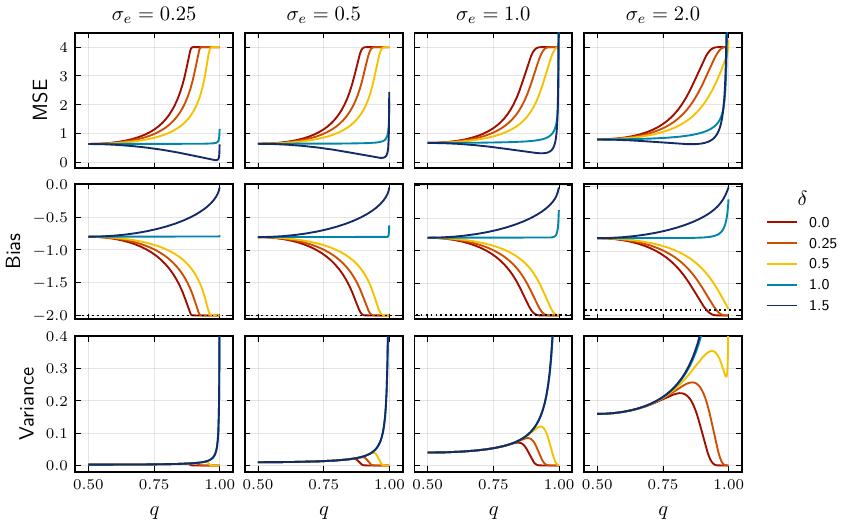}
  \caption{%
    Bias, variance, and mean-squared error for a one-dimensional lasso problem,
    parameterized by noise level (\(\sigma_\varepsilon\)), class balance (\(q\)), and
    scaling (\(\delta\)). Dotted lines represent asymptotic bias of the lasso
    estimator in the case when \(\delta = 1/2\).}
  \label{fig:bias-var-onedim-lasso-full}
\end{figure}

\begin{figure}[htpb]
  \centering
  \includegraphics[]{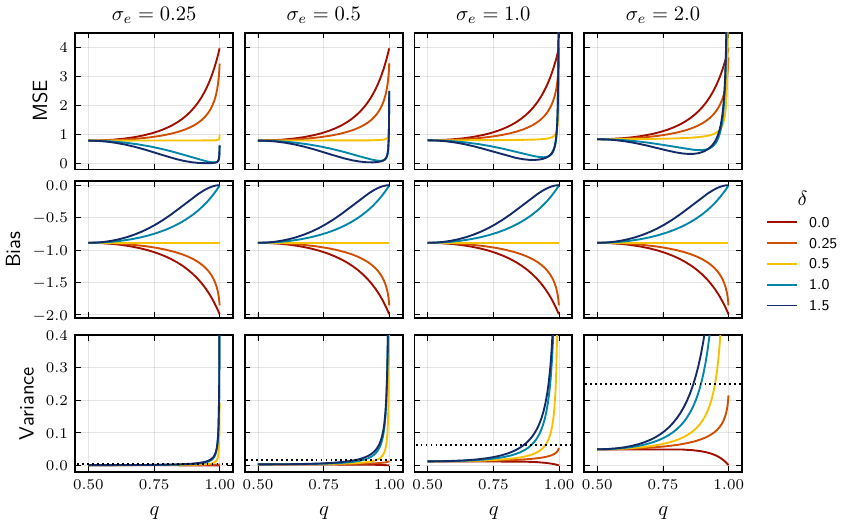}
  \caption{%
    Bias, variance, and mean-squared error for one-dimensional ridge regression,
    parameterized by noise level (\(\sigma_\varepsilon\)), class balance (\(q\)), and
    scaling (\(\delta\)). Dotted lines represent asymptotic bias of the ridge
    estimator in the case of \(\delta = 1/4\).}
  \label{fig:bias-var-onedim-ridge-full}
\end{figure}

\section{Proofs}

\section{Proof of \Cref{thm:classbalance-bias}}%
\label{sec:classbalance-bias-proof}

To avoid excessive notation, we allow ourselves to abuse notation and will drop the
subscript \(j\) everywhere in this proof, allowing \(\beta^*\), \(s\), and so on to
respectively denote \(\beta^*\), \(s_j\) et cetera.

Since \(s = (q - q^2)^\delta\), we have
\begin{align*}
  \mu    & = \beta^* n (q - q^2)^{1 - \delta},                      & \frac{\theta}{\sigma} & = -a \sqrt{q - q^2} - b (q - q^2)^{\delta - 1/2},                                          \\
  \sigma & = \sigma_\varepsilon \sqrt{n} (q - q^2)^{1/2 - \delta},  & \frac{\gamma}{\sigma} & = a \sqrt{q - q^2} - b (q - q^2)^{\delta - 1/2},                                           \\
  d      & = n (q - q^2)^{1 - \delta} + \lambda_2 (q - q^2)^\delta, & \frac{\theta}{d}      & = \frac{-\beta^*n - \lambda_1 (q - q^2)^{\delta - 1}}{n + \lambda_2(q-q^2)^{2\delta - 1}}, \\
  \theta & = -\beta^* n (q - q^2)^{1-\delta} - \lambda_1,           & \frac{\gamma}{d}      & = \frac{\beta^*n - \lambda_1 (q - q^2)^{\delta - 1}}{n + \lambda_2(q-q^2)^{2\delta - 1}},  \\
  \gamma & = \beta^* n (q - q^2)^{1-\delta} - \lambda_1,
\end{align*}
with
\[
  a = \frac{\beta^* \sqrt{n}}{\sigma_\varepsilon} \qquad \text{and} \qquad b = \frac{\lambda_1}{\sigma_\varepsilon \sqrt{n}}.
\]

We are interested in
\begin{equation}
  \label{eq:eval-qlimit}
  \lim_{q \rightarrow 1^-} \E \hat{\beta} =\lim_{q \rightarrow 1^-}\frac{1}{d}\left(-\theta \cdf\left(\frac{\theta}{\sigma}\right) - \sigma \pdf\left(\frac{\theta}{\sigma}\right) + \gamma \cdf\left(\frac{\gamma}{\sigma}\right) + \sigma \pdf\left(\frac{\gamma}{\sigma}\right)\right).
\end{equation}
Before we proceed, note the following limits, which we will make repeated use of throughout the proof.
\begin{equation}
  \label{eq:eval-sigma-limits}
  \lim_{q \rightarrow 1^-} \frac{\theta}{\sigma} = \lim_{q \rightarrow 1^-} \frac{\gamma}{\sigma} =
  \begin{cases}
    -\infty & \text{if } 0 \leq \delta < \frac{1}{2}, \\
    -b      & \text{if } \delta = \frac{1}{2},        \\
    0       & \text{if } \delta > \frac{1}{2},
  \end{cases}
\end{equation}

Starting with the terms involving \(\cdf\) inside the limit in \Cref{eq:eval-qlimit}, for
now assuming that they are well-defined and that the limits of the remaining terms also
exist seperately, we have
\begin{align}
  \lim_{q \rightarrow 1^-} \left(-\frac{\theta}{d} \cdf\left(\frac{\theta}{\sigma}\right) + \frac{\gamma}{d} \cdf \left(\frac{\gamma}{\sigma}\right)\right)
   & = \lim_{q \rightarrow 1^-} \Bigg(\left(\frac{\beta^* n}{n + \lambda_2 (q-q^2)^{2\delta - 1}} + \frac{\lambda_1}{n(q-q^2)^{1-\delta} + \lambda_2 (q-q^2)^{\delta}} \right) \cdf \left(\frac{\theta}{\sigma}\right)  \nonumber                                                                                                                                                                                    \\
   & \phantom{= \lim_{q \rightarrow 1^-} \Bigg( } + \left(\frac{\beta^* n}{n + \lambda_2 (q-q^2)^{2\delta - 1}} - \frac{\lambda_1}{n(q-q^2)^{1-\delta} + \lambda_2 (q-q^2)^{\delta}} \right)\cdf \left(\frac{\gamma}{\sigma}\right) \Bigg) \nonumber                                                                                                                                                                 \\
   & = \lim_{q \rightarrow 1^-} \frac{\beta^*n}{n + \lambda_2 (q-q^2)^{2\delta - 1}}\left(\cdf \left(\frac{\theta}{\sigma}\right) + \cdf \left(\frac{\gamma}{\sigma}\right) \right)                                                                                                                                                                                                                        \nonumber \\
   & \phantom{={}} +  \lim_{q \rightarrow 1^-}\frac{\lambda_1}{n(q-q^2)^{1-\delta} + \lambda_2(q -q^2)^{\delta}} \left(\cdf \left(\frac{\theta}{\sigma}\right) - \cdf \left(\frac{\gamma}{\sigma}\right)\right). \label{eq:eval-qlimit-terms}
\end{align}
Considering the first term in \Cref{eq:eval-qlimit-terms}, we see that
\[
  \lim_{q \rightarrow 1^-} \frac{\beta^*n}{n + \lambda_2 (q-q^2)^{2\delta - 1}}\left(\cdf \left(\frac{\theta}{\sigma}\right) + \cdf \left(\frac{\gamma}{\sigma}\right) \right)  =
  \begin{cases}
    0                                          & \text{if } 0 \leq \delta < 1/2, \\
    \frac{2n \beta^*}{n + \lambda_2} \cdf (-b) & \text{if } \delta = 1/2,        \\
    \beta^*                                    & \text{if } \delta > 1/2.
  \end{cases}
\]
For the second term in \Cref{eq:eval-qlimit-terms}, we start by observing that if \(\delta
= 1\), then \((q - q^2)^{\delta - 1} = 1\), and if \(\delta > 1\), then
\(\lim_{q\rightarrow 1^-}(q - q^2)^{\delta - 1} = 0\). Moreover, the arguments of \(\cdf\)
approach 0 in the limit for \(\delta \geq 1\), which means that the entire term vanishes in
both cases (\(\delta \geq 1\)).

For \(0 \leq \delta < 1\), the limit is indeterminite of the form \(\infty \times 0\). We
define
\[
  f(q) = \cdf \left(\frac{\theta}{\sigma}\right) - \cdf \left(\frac{\gamma}{\sigma}\right)
  \qquad\text{and}\qquad
  g(q) = n(q - q^2)^{1-\delta} + \lambda_2(q - q^2)^\delta,
\]
such that we can express the limit as \(\lim_{q \rightarrow 1^-}f(q)/g(q)\). The
corresponding derivatives are
\[
  \begin{aligned}
    f'(q) & = \left(-\frac{a}{2}(1-2q)(q - q^2)^{-1/2} - b(\delta - 1/2)(1-2q)(q - q^2)^{\delta - 3/2}\right)\pdf\left(\frac{\theta}{\sigma}\right)                \\
          & \phantom{= {}} - \left(\frac{a}{2}(1-2q)(q - q^2)^{-1/2} - b(\delta - 1/2)(1-2q)(q - q^2)^{\delta - 3/2}\right)\pdf\left(\frac{\gamma}{\sigma}\right), \\
    g'(q) & = n(1 - \delta)(1-2q)(q - q^2)^{-\delta} + \lambda_2 \delta(1 - 2q) (q - q^2)^{\delta - 1}
  \end{aligned}
\]
Note that \(f(q)\) and \(g(q)\) are both differentiable and \(g'(q) \neq 0\) everywhere in
the interval \((1/2, 1)\). Now note that we have
\begin{multline}
  \label{eq:eval-qlimit-secondterm}
  \frac{f'(q)}{g'(q)} = \frac{1}{n(1-\delta)(q-q^2)^{1/2-\delta} + \lambda_2 \delta (1-2q)(q - q^2)^{\delta-1/2}} \\
  \times \left(-\left(\frac{a}{2} + b(\delta - 1/2)(q - q^2)^{\delta - 1}\right)\pdf\left(\frac{\theta}{\sigma}\right) - \left(\frac{a}{2} - b(\delta - 1/2)(q - q^2)^{\delta - 1}\right)\pdf\left(\frac{\gamma}{\sigma}\right) \right).
\end{multline}
For \(0 \leq \delta < 1/2\), $\lim_{q \rightarrow 1^-}f'(q)/g'(q) = 0$ since the exponential terms of \(\pdf\) in \Cref{eq:eval-qlimit-secondterm} dominate in the limit.

For \(\delta = 1/2\), we have
\[
  \lim_{q \rightarrow 1^-} \frac{f'(q)}{g'(q)} = -\frac{a}{n + \lambda_2} \lim_{q \rightarrow 1^-}\left(\pdf\left(\frac{\theta}{\sigma}\right) + \pdf\left(\frac{\gamma}{\sigma}\right)\right) = -\frac{2a \pdf(-b)}{n + \lambda_2}
\]
so that we can use L'Hôpital's rule to show that the second term in
\Cref{eq:eval-qlimit-terms} becomes
\begin{equation}
  \label{eq:eval-qlimit-stdcase}
  -\frac{2\beta^*\lambda_1\sqrt{n}}{\sigma_\varepsilon(n + \lambda_2)} \pdf\left(\frac{-\lambda_1}{\sigma_\varepsilon\sqrt{n}}\right).
\end{equation}

For \(\delta > 1/2\), we have
\[
  \begin{aligned}
    \lim_{q \rightarrow 1^-} \frac{f'(q)}{g'(q)} & = \lim_{q \rightarrow 1^-} \frac{-\frac{a}{2}\left(\pdf\left(\frac{\theta}{\sigma}\right) + \pdf\left(\frac{\gamma}{\sigma}\right)\right)}{n(1-\delta)(q - q^2)^{1/2 - \delta} + \lambda_2 \delta(1 - 2q)(q - q^2)^{\delta - 1/2}}                                                       \\
                                                 & \phantom{= {}} + \lim_{q \rightarrow 1^-} \frac{b(\delta - 1/2)\left(\pdf\left(\frac{\gamma}{\sigma}\right) - \pdf\left(\frac{\theta}{\sigma}\right)\right)}{n(1 - \delta)(q - q^2)^{3/2 - 2\delta} + \lambda_2 \delta (1 - 2q)(q - q^2)^{1/2}}                                          \\
                                                 & = 0 + \lim_{q \rightarrow 1^-} \frac{b(\delta - 1/2) e^{-\frac{1}{2}\left(a^2(q - q^2) + b^2(q - q^2)^{2\delta - 1}\right)}\left(e^{-ab(q-q^2)^\delta} - e^{ab(q-q^2)^\delta}\right)}{\sqrt{2\pi}\left(n(1-\delta)(q - q^2)^{3/2 - 2\delta} + \lambda_2\delta(1-2q)(q-q^2)^{1/2}\right)} \\
                                                 & = 0
  \end{aligned}
\]
since the exponential term in the numerator dominates.

Now we proceed to consider the terms involving \(\pdf\) in \Cref{eq:eval-qlimit}. We have
\begin{equation}
  \label{eq:eval-qlimit-pdfterm}
  \lim_{q \rightarrow 1^-} \frac{\sigma}{d} \left(\pdf \left(\frac{\gamma}{\sigma}\right) - \pdf \left(\frac{\theta}{\sigma}\right)\right)
  = \sigma_\varepsilon \sqrt{n} \lim_{q \rightarrow 1^-} \frac{ \pdf\left(\frac{\gamma}{\sigma}\right) - \pdf\left(\frac{\theta}{\sigma}\right)}{n(q-q^2)^{1/2} + \lambda_2(q - q^2)^{2\delta - 1/2}}
\end{equation}
For \(0 \leq \delta < 1/2\), we observe that the exponential terms in \(\pdf\) dominate in the limit, and so we can distribute the limit and consider the limits of the respective terms individually, which both vanish.

For \(\delta \geq 1/2\), the limit in \Cref{eq:eval-qlimit-pdfterm} has an indeterminate
form of the type \(\frac{0}{0}\). Define
\[
  u(q) = \pdf\left(\frac{\gamma}{\sigma}\right) - \pdf\left(\frac{\theta}{\sigma}\right)
  \qquad\text{and}\qquad
  v(q) = n(q - q^2)^{1/2} + \lambda_2 (q - q^2)^{2\delta - 1/2}
\]
which are both differentiable in the interval \((1/2, 1)\) and \(v'(q) \neq 0\) everywhere
in this interval. The derivatives are
\[
  \begin{aligned}
    u'(q) & = -\pdf\left(\frac{\gamma}{\sigma}\right)\frac{\gamma}{\sigma} \left(\frac{1}{2}\left(a(1-2q)(q - q^2)^{-1/2}\right) - b(\delta - 1/2)(1- 2q)(q - q^2)^{\delta - 3/2}\right)                 \\
          & \phantom{= {}} + \pdf\left(\frac{\theta}{\sigma}\right) \frac{\theta}{\sigma} \left(\frac{1}{2}\left(a(1-2q)(q - q^2)^{-1/2}\right) + b(\delta - 1/2)(1- 2q)(q - q^2)^{\delta - 3/2}\right), \\
    v'(q) & = \frac{n}{2} (1 - 2q)(q - q^2)^{-1/2} + \lambda_2(2\delta - 1/2)(1 - 2q)(q - q^2)^{2\delta - 3/2}.
  \end{aligned}
\]
And so
\begin{equation}
  \begin{split}
    \frac{u'(q)}{v'(q)} = \frac{1}{n + \lambda_2(4\delta - 1)(q - q^2)^{2\delta - 1}}  \Bigg( & \left(a - b(2\delta - 1)(q - q^2)^{\delta - 1}\right) \pdf\left(\frac{\gamma}{\sigma}\right) \frac{\gamma}{\sigma} \\ & + \left(a + b(2\delta - 1)(q - q^2)^{\delta - 1}\right)\pdf\left(\frac{\theta}{\sigma}\right) \frac{\theta}{\sigma}\Bigg).
  \end{split}
\end{equation}
Taking the limit, rearranging, and assuming that the limits of the separate terms exist, we obtain
\begin{multline}
  \label{eq:eval-qlimit-pdfterm-split}
  \lim_{q \rightarrow 1^-} \frac{u'(q)}{v'(q)} = a \lim_{q \rightarrow 1^-}  \frac{1}{n + \lambda_2 (4 \delta - 1)(q - q^2)^{2\delta - 1}} \left( \pdf\left(\frac{\gamma}{\sigma}\right)\frac{\gamma}{\sigma} - \pdf\left(\frac{\theta}{\sigma}\right)\frac{\theta}{\sigma}\right) \\
  + b (2\delta - 1) \lim_{q \rightarrow 1^-} \frac{1}{n + \lambda_2 (4 \delta - 1)(q - q^2)^{2\delta - 1}} \bigg( \pdf\left(\frac{\gamma}{\sigma}\right) \left(a(q -q^2)^{\delta - 1/2}- b(q - q^2)^{2\delta - 3/2} \right) \\
  - \pdf\left(\frac{\theta}{\sigma}\right) \left(-a(q - q^2)^{\delta - 1/2} - b(q - q^2)^{2\delta - 3/2}\right) \bigg).
\end{multline}
For \(\delta = 1/2\), we have
\[
  \lim_{q \rightarrow 1^-} \frac{u'(q)}{v'(q)} = -\frac{a}{n + \lambda_2}\left(-b \pdf(-b) - b \pdf(-b)\right) + 0 = 2ab \pdf(-b) = \frac{2 \beta^* \lambda_1}{\sigma_\varepsilon^2(n + \lambda_2)} \pdf \left(\frac{-\lambda_1}{\sigma_\varepsilon\sqrt{n}}\right).
\]
Using L'Hôpital's rule, \Cref{eq:eval-qlimit-pdfterm} must consequently be
\[
  \frac{2 \beta^* \lambda_1\sqrt{n}}{\sigma_\varepsilon(n + \lambda_2)} \pdf \left(\frac{-\lambda_1}{\sigma_\varepsilon\sqrt{n}}\right),
\]
which cancels with \Cref{eq:eval-qlimit-stdcase}.

For \(\delta > 1/2\), we first observe that the first term in
\Cref{eq:eval-qlimit-pdfterm-split} tends to zero due to \Cref{eq:eval-sigma-limits} and
the properties of the standard normal distribution. For the second term, we note that this
is essentially of the same form as \Cref{eq:eval-qlimit-secondterm} and that the limit is
therefore 0 here.

\subsection{Proof of \Cref{thm:classbalance-variance}}\label{sec:classbalance-variance-proof}

The variance of the elastic net estimator is given by
\begin{multline}
  \label{eq:varthm-var}
  \var \hat{\beta}_j = \frac{1}{d^2}\Bigg( \frac{\sigma^2}{2}\bigg(2 + \erf\left(\frac{\theta}{\sigma \sqrt{2}}\right) - \frac{\theta}{\sigma}\sqrt{\frac{2}{\pi}} \exp\left(-\frac{\theta^2}{2\sigma^2}\right) + \erf\left(\frac{\gamma}{\sigma\sqrt{2}}\right) - \frac{\gamma}{\sigma} \sqrt{\frac{2}{\pi}} \exp\left(- \frac{\gamma^2}{2\gamma^2}\right)\bigg) \\
  + 2\theta\sigma \pdf\left(\frac{\theta}{\sigma}\right) + \theta^2 \cdf\left(\frac{\theta}{\sigma}\right) + 2\gamma \sigma \pdf\left(\frac{\gamma}{\sigma}\right) + \gamma^2 \cdf\left(\frac{\gamma}{\sigma}\right) \Bigg)
  - \left(\frac{1}{d}\E \hat{\beta}_j\right)^2.
\end{multline}
We start by noting the following identities:
\[
  \begin{aligned}
    \theta^2                  & = \left(\beta^* n\right)^2 (q-q^2)^{2-2\delta} + \lambda_1^2 + 2\lambda_1 \beta^* n(q-q^2)^{1-\delta},                \\
    d^2                       & = n^2(q -q^2)^{2 - 2\delta} + 2n\lambda_2 (q-q^2) + \lambda_2^2 (q-q^2)^{2\delta},                                    \\
    \theta \sigma             & =  -\sigma_\varepsilon\left(\beta^* n^{3/2}(q- q^2)^{3/2-2\delta} + \sqrt{n} \lambda_1 (q-q^2)^{1/2 - \delta}\right), \\
    \frac{\theta^2}{\sigma^2} & = a^2(q-q^2) + b^2(q-q^2)^{2\delta - 1} + 2ab (q -q^2)^\delta,                                                        \\
    \frac{\sigma}{d}          & = \frac{\sigma_\varepsilon \sqrt{n}}{n(q-q^2)^\frac{1}{2} + \lambda_2 (q-q^2)^{2\delta - 1/2}}.
  \end{aligned}
\]
Expansions involving \(\gamma\) instead of \(\theta\) have identical expansions up to sign
changes of the individual terms. Also recall the definitions provided in the proof of
\Cref{thm:classbalance-bias}.

Starting with the case when \(0 \leq \delta < 1/2\), we write the limit of
\Cref{eq:varthm-var} as
\begin{align*}
   & \lim_{q \rightarrow 1^-} \var \hat{\beta}_j                                                                                                                                                                                                                                                                                                                                         \\ & = \sigma_\varepsilon^2 n  \lim_{q \rightarrow 1^-} \frac{1}{\left(n(q-q^2)^{1/2} + \lambda_2(q-q^2)^{2\delta - 1/2}\right)^2}\bigg(1 + \erf\left(\frac{\theta}{\sigma \sqrt{2}}\right) - \frac{\theta}{\sigma}\sqrt{\frac{2}{\pi}} \exp\left(-\frac{\theta^2}{2\sigma^2}\right)\bigg)                                                                                               \\
   & \phantom{= {}} + \sigma_\varepsilon^2 n  \lim_{q \rightarrow 1^-} \frac{1}{\left(n(q-q^2)^{1/2} + \lambda_2(q-q^2)^{2\delta - 1/2}\right)^2}\bigg(1 + \erf\left(\frac{\gamma}{\sigma\sqrt{2}}\right) - \frac{\gamma}{\sigma} \sqrt{\frac{2}{\pi}} \exp\left(- \frac{\gamma^2}{2\sigma^2}\right)\bigg)                                                                               \\
   & \phantom{= {}}+ \lim_{q \rightarrow 1^-} \frac{2\theta\sigma}{d^2} \pdf\left(\frac{\theta}{\sigma}\right) + \lim_{q \rightarrow 1^-} \frac{\theta^2}{d^2} \cdf\left(\frac{\theta}{\sigma}\right) + \lim_{q \rightarrow 1^-} \frac{2\gamma}{d^2} \sigma \pdf\left(\frac{\gamma}{\sigma}\right) + \lim_{q \rightarrow 1^-}\frac{\gamma^2}{d^2} \cdf\left(\frac{\gamma}{\sigma}\right) \\
   & \phantom{= {}}- \left( \lim_{q\rightarrow 1^-}\frac{1}{d}\E \hat{\beta}_j\right)^2,
\end{align*}
assuming, for now, that all limits exist. Next, let
\[
  \begin{aligned}
    f_1(q) & = 1 + \erf\left(\frac{\theta}{\sigma\sqrt{2}}\right) - \frac{\theta}{\sigma}\sqrt{\frac{2}{\pi}} \exp\left(-\frac{\theta^2}{2\sigma^2}\right) , \\
    f_2(q) & = 1 + \erf\left(\frac{\gamma}{\sigma\sqrt{2}}\right) - \frac{\gamma}{\sigma}\sqrt{\frac{2}{\pi}} \exp\left(-\frac{\gamma^2}{2\sigma^2}\right) , \\
    g(q)   & = \left(n^2(q-q^2) + 2n \lambda_2 (q-q^2)^{2\delta} + \lambda_2^2 (q-q^2)^{4\delta - 1}\right)^2.
  \end{aligned}
\]
And
\begin{align*}
  f_1'(q) & = \frac{\theta^2}{\sigma^2}\sqrt{\frac{2}{\pi}}\exp\left(-\frac{\theta^2}{2\sigma^2}\right),                                    \\
  f_2'(q) & = \frac{\gamma^2}{\sigma^2}\sqrt{\frac{2}{\pi}}\exp\left(-\frac{\gamma^2}{2\sigma^2}\right),                                    \\
  g'(q)   & = (1-2q)\left((q-q^2)^{-1} + 4n\delta \lambda_2 (q-q^2)^{2\delta - 1} + \lambda_2^2 (4 \delta - 1)(q-q^2)^{4\delta - 2}\right).
\end{align*}
\(f_1\), \(f_1\) and \(g\) are differentiable in \((1/2, 1)\) and \(g'(q) \neq 0\) everywhere in this interval. \(f_1/g\) and \(f_2/g\) are indeterminate of the form \(0/0\). And we see that
\[
  \lim_{q \rightarrow 1^-} \frac{f_1'(q)}{g'(q)} = \lim_{q \rightarrow 1^-} \frac{f_2'(q)}{g'(q)} = 0
\]
due to the dominance of the exponential terms as \(\theta/\sigma\) and \(\gamma/\sigma\)
both tend to \(-\infty\). Thus \(f_1/g\) and \(f_2/g\) also tend to 0 by L'Hôpital's rule.
Similar reasoning shows that
\[
  \lim_{q \rightarrow 1^-} \frac{2\theta \sigma}{d^2} \pdf \left(\frac{\theta}{\sigma}\right) = \lim_{q \rightarrow 1^-} \frac{\theta^2}{d^2} \cdf \left(\frac{\theta}{\sigma}\right) = 0.
\]
The same result applies to the respective terms involving \(\gamma\). And since we in
\Cref{thm:classbalance-bias} showed that \(\lim_{q\rightarrow 1^-} \frac{1}{d} \E
\hat{\beta}_j = 0\), the limit of \Cref{eq:varthm-var} must be 0.

For \(\delta = 1/2\), we start by establishing that
\begin{equation}
  \label{eq:var-deltahalf-integral}
  \lim_{q \rightarrow 1^-} \int_{-\infty}^{-\lambda}(z+ \lambda)^2 f_Z(z) \du z = \lim_{q \rightarrow 1^-} \left(\sigma^2 \int_{-\infty}^\frac{\theta}{\sigma} y^2 \pdf(y) \du y + 2 \theta \sigma \int_{-\infty}^\frac{\theta}{\sigma} y \pdf(y) \du y + \theta^2 \int_{-\infty}^\frac{\theta}{\sigma} \pdf(y) \du y\right)
\end{equation}
is a positive constant since \(\theta/\sigma \rightarrow -b\), \(\sigma =
\sigma_\varepsilon \sqrt{n}\), \(\theta \rightarrow - \lambda\), and \(\theta\sigma
\rightarrow - \sigma_\varepsilon \sqrt{n}\lambda\). An identical argument can be made in
the case of \(\lim_{q \rightarrow 1^-} \int_{\lambda}^{\infty}(z - \lambda)^2 f_Z(z) \du
z.\) We then have
\[
  \lim_{q \rightarrow 1^-} \frac{1}{d^2} \int_{-\infty}^{-\lambda}(z+ \lambda)^2 f_Z(z) \du z = \frac{C^+}{\lim_{q\rightarrow 1^-} d^2} = \frac{C^+}{0} = \infty,
\]
where \(C^+\) is some positive constant. And because \(\lim_{q\rightarrow 1^-} \frac{1}{d}
\E \hat{\beta}_j = \beta^*\)~(\Cref{thm:classbalance-bias}), the limit of
\Cref{eq:varthm-var} must be \(\infty\).

Finally, for the case when \(\delta > 1/2\), we have
\begin{multline}
  \label{eq:var-deltagthalf-integral}
  \lim_{q \rightarrow 1^-} \frac{1}{d^2} \left(\sigma^2 \int_{-\infty}^\frac{\theta}{\sigma} y^2 \pdf(y) \du y + 2 \theta \sigma \int_{-\infty}^\frac{\theta}{\sigma} y \pdf(y) \du y + \theta^2 \int_{-\infty}^\frac{\theta}{\sigma} \pdf(y) \du y\right) \\
  = \lim_{q \rightarrow 1^-} \Bigg( \frac{n \sigma^2}{ \left(n (q-q^2)^{1/2} + \lambda_2(q-q^2)^{2\delta - 1/2}\right)^2} \int_{-\infty}^\frac{\theta}{\sigma} y^2 \pdf(y) \du y
  \\- \frac{2\sigma_\varepsilon \sqrt{n}\left(\beta^* n (q- q^2)^{1 - \delta} - \lambda_1\right)}{\left(n(q-q^2)^{3/4 - \delta/2} + \lambda_2(q-q^2)^{3\delta / 2 - 1/4}\right)^2} \int_{-\infty}^\frac{\theta}{\sigma} y \pdf(y) \du y \\
  + \left(\frac{-\beta^*n(q-q^2)^{1-\delta} - \lambda_1}{n(q-q^2)^{1-\delta} + \lambda_2(q-q^2)^\delta}\right)^2 \int_{-\infty}^\frac{\theta}{\sigma} \pdf(y) \du y\Bigg).
\end{multline}
Inspection of the exponents involving the factor \((q - q^2)\) shows that the first term inside the limit will dominate. And since the upper limit of the integrals, \(\theta/\sigma \rightarrow  0\)  as \(q \rightarrow 1^-\), the limit must be \(\infty\).

\subsection{Proof of \Cref{cor:ridge-variance}}

We have
\begin{equation*}
  \lim_{q\rightarrow 1^-}\var \hat{\beta}_j = \lim_{q \rightarrow 1^-}\frac{\sigma^2}{d^2} \left(\frac{\sigma_\varepsilon \sqrt{n} (q - q^2)^{1/2 - \delta}}{n (q-q^2)^{1 - \delta} + \lambda_2 (q-q^2)^\delta}\right)^2
  = \frac{\sigma_\varepsilon^2 n}{\lambda_2^2} \lim_{q \rightarrow 1^-}(q-q^2)^{1 - 4\delta},
\end{equation*}
from which the result follows directly.

\subsection{Proof of \Cref{thm:weighted-elasticnet-bias-variance}}

\subsubsection{Expected Value}

Starting with the expected value, our proof follows a similar structure as in the proof for
\Cref{thm:classbalance-bias} (\Cref{sec:classbalance-bias-proof}). We start by noting the
values of some of the important terms. As before we will drop the subscript \(j\)
everywhere to simplify notation. We have
\begin{align*}
  \mu    & = \beta^*(q -q^2)^\omega,                     & \frac{\theta}{\sigma} & = -a\sqrt{q-q^2} - b(q-q^2)^{\omega - 1/2},                                           \\
  \sigma & = \sigma_\varepsilon \sqrt{n(q-q^2)},         & \frac{\gamma}{\sigma} & = a\sqrt{q-q^2} - b(q-q^2)^{\omega - 1/2},                                            \\
  d      & = n(q-q^2) + \lambda_2(q-q^2)^\omega.\,       & \frac{\theta}{d}      & = \frac{-\beta^*n - \lambda_1(q-q^2)^{\omega - 1}}{n + \lambda_2(q-q^2)^{\omega -1}}, \\
  \theta & = -\beta^*n(q-q^2) - \lambda_1(q-q^2)^\omega, & \frac{\gamma}{d}      & = \frac{\beta^*n - \lambda_1(q-q^2)^{\omega - 1}}{n + \lambda_2(q-q^2)^{\omega -1}},  \\
  \gamma & = \beta^*n(q-q^2) - \lambda_1(q-q^2)^\omega.
\end{align*}

First note the following limit (which is analogous to that in \Cref{eq:eval-sigma-limits}).
\begin{equation}
  \lim_{q \rightarrow 1^-} \frac{\theta}{\sigma} = \lim_{q \rightarrow 1^-} \frac{\gamma}{\sigma} =
  \begin{cases}
    -\infty & \text{if } 0 \leq \omega < \frac{1}{2}, \\
    -b      & \text{if } \omega = \frac{1}{2},        \\
    0       & \text{if } \omega > \frac{1}{2}.
  \end{cases}
\end{equation}

As in \Cref{sec:classbalance-bias-proof}, we are looking to compute the following limit:
\begin{equation}
  \label{eq:eval-qlimit-weighted}
  \lim_{q \rightarrow 1^-} \E \hat{\beta} =\lim_{q \rightarrow 1^-}\frac{1}{d}\left(-\theta \cdf\left(\frac{\theta}{\sigma}\right) - \sigma \pdf\left(\frac{\theta}{\sigma}\right) + \gamma \cdf\left(\frac{\gamma}{\sigma}\right) + \sigma \pdf\left(\frac{\gamma}{\sigma}\right)\right).
\end{equation}

Starting with the terms involving \(\cdf\) and assuming that the limit can be distributed,
we have
\begin{align}
  \lim_{q \rightarrow 1^-} \left(-\frac{\theta}{d} \cdf\left(\frac{\theta}{\sigma}\right) + \frac{\gamma}{d} \cdf \left(\frac{\gamma}{\sigma}\right)\right)
   & = \lim_{q\rightarrow 1^-} \frac{\beta^* n + \lambda_1(q-q^2)^{\omega -1}}{n + \lambda_2(q-q^2)^{\omega - 1}}\cdf\left(\frac{\theta}{\sigma}\right) \nonumber                \\
   & \phantom{= {}} + \lim_{q\rightarrow 1^-} \frac{\beta^* n - \lambda_1(q-q^2)^{\omega -1}}{n + \lambda_2(q-q^2)^{\omega - 1}}\cdf\left(\frac{\gamma}{\sigma}\right) \nonumber \\
   & =\begin{cases}
        0                              & \text{if } 0 \leq \omega < 1, \\
        \frac{\beta^*n}{n + \lambda_2} & \text{if } \omega = 1,        \\
        \beta^*                        & \text{if } \omega > 1.
      \end{cases} \label{eq:eval-qlimit-terms-weighted}
\end{align}

The derivation of the first case in \Cref{eq:eval-qlimit-terms-weighted} depends on
\(\omega\). For \(0 \leq \omega \leq 1/2\), it stems from the facts that
\(\cdf(\theta/\sigma) \rightarrow 0\) and \(\cdf(\theta/\sigma) \rightarrow 0\) as \(q
\rightarrow 1^-\) together with the existence of the \((q-q^2)^{\omega - 1}\) factor in
both numerator and denominator. For \(1/2 \leq \omega < 1\), the terms cancel each other
out. In the second case, when \(\omega = 1\), the result stems from \(\cdf(\theta/\sigma)\)
and \(\cdf(\gamma/\sigma)\) both tending to 1/2 as \(q \rightarrow 1^-\). And finally for
\(\omega > 1\), the terms involving the \((q-q^2)^{\omega - 1}\) factors vanish and again
the values of the cumulative distribution functions tend to 1/2.

Now, we turn to the terms involving the probability density function \(\pdf\). Again, we
assume the limit is distributive so that
\begin{equation}
  \label{eq:eval-qlimit-pdf-weighted}
  \lim_{q \rightarrow 1^-} \frac{\sigma}{d} \left(\pdf\left(\frac{\gamma}{\sigma}\right) - \left(\frac{\theta}{\sigma}\right)\right) =
  \lim_{q\rightarrow 1^-}\frac{\sigma}{d} \pdf \left(\frac{\gamma}{\sigma}\right) - \lim_{q \rightarrow 1^-} \frac{\sigma}{d} \pdf \left(\frac{\theta}{\sigma}\right).
\end{equation}
Starting with the first term on the right-hand side of \Cref{eq:eval-qlimit-pdf-weighted},
we have
\[
  \lim_{q\rightarrow 1^-}\frac{\sigma}{d} \pdf \left(\frac{\gamma}{\sigma}\right) =
  \frac{\sigma_\varepsilon \sqrt{n} \pdf\left(\frac{\gamma}{\sigma}\right)}{n(q-q^2)^{1/2} + \lambda_2(q-q^2)^{\omega - 1/2}}.
\]
For \(0 \leq \omega < 1/2\), this limit is 0 since the exponential terms in the numerator
will dominate as \(q \rightarrow 1^-\). For \(\omega = 1/2\), we have the limit
\(\sigma_\varepsilon\sqrt{n} \pdf(-b)/\lambda_2\). For \(\omega > 1/2\), the limit is
indeterminate of the type \(0/0\). Let
\begin{equation*}
  f_1(q) = \pdf\left(\frac{\gamma}{\sigma}\right)
  \qquad\text{and}\qquad
  g(q) = n(q-q^2)^{1/2} + \lambda_2(q-q^2)^{\omega - 1/2}
\end{equation*}
and observe that \(f_1\) and \(g\) are differentiable and \(g'(q) \neq 0 \) for \(q \in
(1/2, 1)\). The derivatives are
\begin{align*}
  f_1'(q) & = -\left(\frac{a}{2}(1-2q)(q-q^2)^{-1/2} - b(\omega - 1/2)(1-2q)(q-q^2)^{\omega - 3/2}\right) \frac{\gamma}{\sigma} \pdf\left(\frac{\gamma}{\sigma}\right), \\
  g'(q)   & = \frac{n}{2}(1-2q)(q-q^2)^{-1/2} + \lambda_2(\omega - 1/2)(1-2q)(q-q^2)^{\omega - 3/2}.
\end{align*}
Next, we find that
\begin{equation}
  \label{eq:eval-qlimit-pdf-derivative-weighted}
  \frac{f_1'(q)}{g'(q)} = \frac{-a + b(2\omega - 1)(q-q^2)^{\omega - 1}}{n + \lambda_2(2\omega - 1)(q-q^2)^{\omega - 1}} \left(\frac{\gamma}{\sigma}\right) \pdf\left(\frac{\gamma}{\sigma}\right).
\end{equation}
Taking the limit of \Cref{eq:eval-qlimit-pdf-derivative-weighted} and invoking L'Hôpital's
rule yields
\[
  \lim_{q\rightarrow 1^-} \frac{f_1'(q)}{g'(q)} =0
\]
both when \(1/2 < \omega < 1\) and \(\omega \geq 1\) since \(\gamma/\sigma\) tends to 0 as
\(q \rightarrow 1^-\) for \(\omega > 1/2\) and the \(\pdf\) term tends to a constant, plus
the fact that the remaining factor in the expression also tends to a constant since the
terms involving \((q-q^2)^{\omega - 1}\) vanish when \(\omega > 1\), are constant when
\(\omega = 1\), and cancel each other out in the limit when \(\omega < 1\).

Finally, if we now consider the second term on the right-hand side of
\Cref{eq:eval-qlimit-pdf-weighted}, set \(f_2(q) = \pdf(\theta/\sigma)\), and perform the
same steps as above, we find that the limits are the same in all cases, which means that
the limits in \Cref{eq:eval-qlimit-pdf-weighted} cancel in the case when \(\omega = 1/2\)
and therefore that
\[
  \lim_{q \rightarrow 1^-} \frac{\sigma}{d} \left(\pdf\left(\frac{\gamma}{\sigma}\right) - \left(\frac{\theta}{\sigma}\right)\right) = 0
\]
for \(0 \leq \omega \). The limit in \Cref{eq:eval-qlimit-weighted} is given by
\Cref{eq:eval-qlimit-terms-weighted}.

\subsubsection{Variance}

The proof for the variance result is in many ways equivalent to that in the case of
variance of the normalized unweighted elastic net (\Cref{sec:classbalance-variance-proof})
and we therefore omit many of the details here.

In the case of \(0 \leq \omega < 1/2\), the proof is simplified since the \(d^2\) term
tends to \(\infty\) whilst the numerator takes the same limit as in the normalized case,
which means that the limit is \(0\) in this case. For \(\omega = 1/2\), we consider
\Cref{eq:var-deltahalf-integral} and observe that it again tends to a positive constant
whilst \(\lim_{q\rightarrow 1^-} d^2 = 0\), which means that the limit of the expression,
and hence variance of the estimator, tends to \(\infty\). For \(\omega > 1/2\), an
identical argument for the expression in \Cref{eq:var-deltagthalf-integral} holds and the
limit is therefore \(\infty\) in this case as well.

\subsection{Proof of \Cref{thm:maxabs-gev}}


If \(X_i \sim \normal(\mu, \sigma)\), then \(|X_i| \sim \fnormal(\mu,\sigma)\). By the
Fisher--Tippett--Gnedenko theorem, we know that \((\max_i |X_i| - b_n) / a_n\) converges in
distribution to either the Gumbel, Fréchet, or Weibull distribution, given a proper choice
of \(a_n > 0\) and \(b_n \in \mathbb{R}\). A sufficient condition for convergence to the
Gumbel distribution for a absolutely continuous cumulative distribution
function~\citep[Theorem 10.5.2]{nagaraja2003} is
\[
  \lim_{x \rightarrow \infty} \frac{d}{dx}\left(\frac{1- F(x)}{f(x)}\right) = 0.
\]
We have
\[
  \begin{aligned}
    \frac{1 - F_Y(x)}{f_Y(x)} & = \frac{1 - \frac{1}{2}\erf{\left(\frac{x - \mu}{\sqrt{2\sigma^2}}\right)} - \frac{1}{2}\erf{\left(\frac{x + \mu}{\sqrt{2\sigma^2}}\right)}}{\frac{1}{\sqrt{2\pi\sigma^2}}e^{\frac{-(x-\mu)^2}{2\sigma^2}} + \frac{1}{\sqrt{2\pi\sigma^2}}e^{\frac{-(x+\mu)^2}{2\sigma^2}}} \\
                              & = \frac{2 - \cdf\left(\frac{x - \mu}{\sigma}\right) - \cdf\left(\frac{x + \mu}{\sigma}\right)}{\frac{1}{\sigma}\left(\pdf\left(\frac{x - \mu}{\sigma}\right) + \pdf\left(\frac{x + \mu}{\sigma}\right)\right)}                                                              \\
                              & \rightarrow \frac{\sigma(1 - \cdf(x))}{\pdf(x)} \text{ as } n \rightarrow n,
  \end{aligned}
\]
where \(\pdf\) and \(\cdf\) are the probability distribution and cumulative density
functions of the standard normal distribution respectively. Next, we follow \citet[example
  10.5.3]{nagaraja2003} and observe that
\[
  \frac{d}{dx} \frac{\sigma(1 - \cdf(x))}{\pdf(x)} = \frac{\sigma x (1 - \cdf(x))}{\pdf(x)} - \sigma \rightarrow 0 \text{ as } x \rightarrow \infty
\]
since
\[
  \frac{1 - \cdf(x)}{\pdf(x)} \sim \frac{1}{x}.
\]
In this case, we may take \(b_n = F_Y^{-1}(1 - 1/n)\) and \(a_n = \big(n
f_Y(b_n)\big)^{-1}\).

\section{Additional Experiments}
\label{sec:additional-experiments}

In this section we present additional and extended results from the main section.

\subsection{Power and False Discoveries for Multiple Features}%
\label{sec:power-fdr-multiple}

Here, we study how the power of correctly detecting \(k=10\) signals under \(q_j\) linearly
spaced in \([0.5, 0.99]\)~(\Cref{fig:binary-power}). We set \(\beta^*_j = 2\) for each of
the signals, use \(n = 100\,000\), and let \(\sigma_\varepsilon = 1\). The level of
regularization is set to \(\lambda_1 = n 4^\delta/10\). As we can see, the power is
directly related to \(q_j\) and for unbalanced features stronger the higher the choice of
\(\delta\) is.

We also consider a version of the same setup, but with \(p\) linearly spaced in \([20,
    100]\) and compute normalized mean-squared error (NMSE) and false discovery rate
(FDR)~(\Cref{fig:binary-fdr-mse}). As before, we let \(k = 10\) and consider three
different levels of class imbalance. The remaining \(p-k\) features have class balances
spaced evenly on a logarithmic scale from 0.5 to 0.99. Unsurprisingly, the increase in
power gained from selecting \(\delta = 1\) imposes increased false discovery rates. We also
see that the mean-squared error depends on class balance. In line with our previous
results, \(\delta \in \{0, 1/2\}\) appears to work well for balanced features whilst
\(\delta = 1\) works better when there are large imbalances. In the case when \(q_j =
0.99\), the model under scaling with \(\delta = 0\) does not detect any of the true
signals.

\begin{figure}[htpb]
  \centering
  \subfigure[%
    The power (probability of detecting all true signals) of the lasso. In our
    orthogonal setting, power is constant over \(p\), which is why we have
    omitted the parameter in the plot. \label{fig:binary-power}
  ]{\includegraphics[]{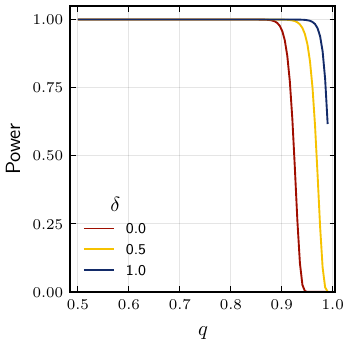}}\hfill%
  \subfigure[%
    NMSE and FDR: the rate of coefficients incorrectly set to non-zero (false
    discoveries) to the total number of estimated coefficients that are nonzero
    (discoveries).
    \label{fig:binary-fdr-mse}]{\includegraphics[]{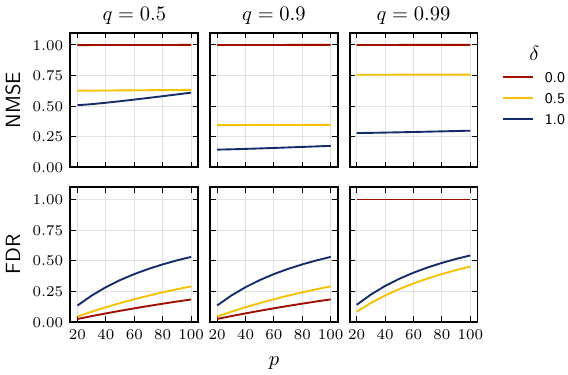}}
  \caption{%
    Normalized mean-squared error (NMSE), false discovery rate (FDR), and power
    for a lasso problem with \(k = 10\) true signals (nonzero \(\beta_j^*\)),
    varying \(p\), and \(q_j \in [0.5, 0.99]\). The noise level is set at
    \(\sigma_\varepsilon = 1\) and \(\lambda_1 = 0.02\).
  }
\end{figure}

\subsection{Support Size and Predictive Performance}
\label{sec:predictive-performance-support}

In this section we analyze the support size of the lasso estimates for the experiment in
\Cref{sec:experiments-predictive-performance}. In \Cref{fig:hyperopt-support}, we have, in
addition to NMSE on the validation set, also plotted the size of the support of the lasso
(cardinality of the set of features that have corresponding nonzero coefficients). Here we
only show results for \(\delta \in \{0, 1/2, 1\}\).

\begin{figure}[htpb]
  \centering
  \includegraphics[]{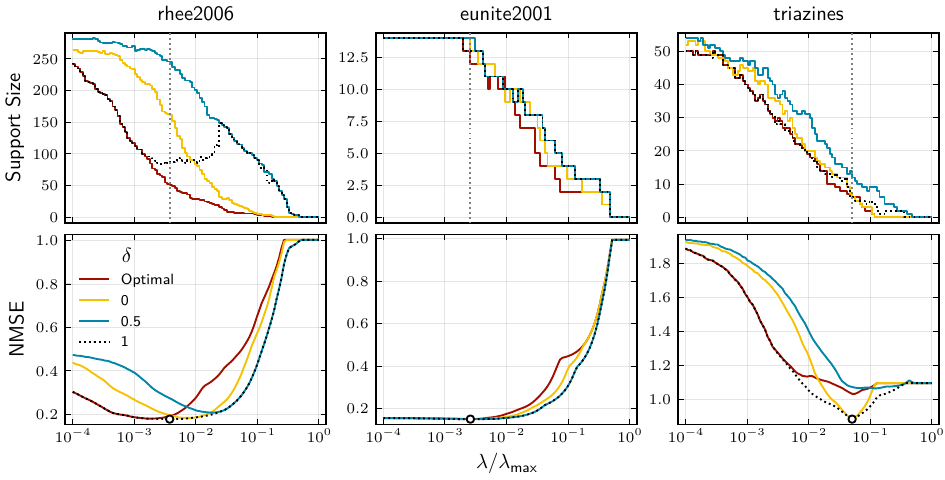}
  \caption{%
    Support size and normalized mean-squared error (NMSE) for the validation set for the lasso
    fit to datasets \data{rhee2006}, \data{eunite2001}, and \data{triazines} across combinations of
    \(\delta\) and \(\lambda\). The optimal \(\delta\) is marked with dashed black lines and
    the best combination of \(\delta\) (among 0, 1/2, and 1) and \(\lambda\) is shown as a dot.
  }
  \label{fig:hyperopt-support}
\end{figure}

\subsection{Predictive Performance for Simulated Data}%
\label{sec:predictive-performance-simulated}

In this experiment, we consider predictive performance in terms of mean-squared error of
the lasso and ridge regression given different levels of class balance (\(q_j \in \{0.5,
0.9, 0.99\}\)), signal-to-noise ratio, and normalization (\(\delta\)). All of the features
are binary, but here we have used \(n=300\) and \(p = \num{1000}\). The \(k=10\) first
features correspond to true signals with \(\beta^*_j = 1\) and all have class balance
\(q\). To set signal-to-noise ratio levels, we rely on the same choice as in
\citet{hastie2020} and use a log-spaced sequence of values from 0.05 to 6. We use standard
hold-out validation with equal splits for training, validation, and test sets. And we fit a
full lasso path, parameterized by a log-spaced grid of 100 values\footnote{This is a
  standard choice of grid, used for instance by \citet{friedman2010}}, from
\(\lambda_\text{max}\) (the value of \(\lambda\) at which the first feature enters the
model) to \(10^{-2}\lambda_\text{max}\) on the training set and pick \(\lambda\) based on
validation set error. Then we compute the hold-out test set error and aggregate the results
across 100 iterations.

The results~(\Cref{fig:binary-sim}) show that the optimal normalization type in terms of
prediction power depends on the class balance of the true signals. If the imbalance is
severe, then we gain by using \(\delta=1/2\) or \(1\), which gives a chance of recovering
the true signals. If everything is balanced, however, then we do better by not scaling. In
general, \(\delta=1/2\) works well for these specific combinations of settings.

\begin{figure}[htpb]
  \centering
  \includegraphics[]{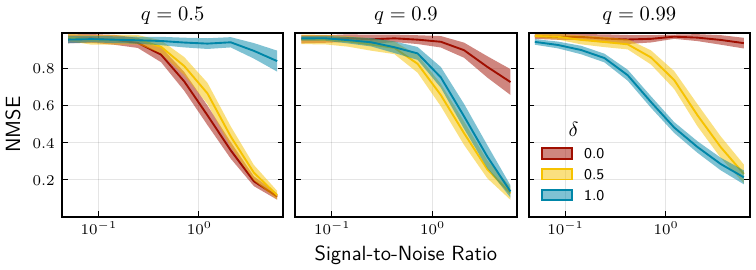}
  \caption{%
    Normalized mean-squared prediction error in a lasso model for different types of
    normalization (\(\delta\)), types of class imbalances (\(q_j\)), and signal-to-noise ratios
    (0.05 to 6) in a dataset with \(n=300\) observations and \(p = \num{1000}\) features. The
    error is aggregated test-set error from hold-out validation with \(100\) observations in
    each of the training, validation, and test sets. The plot shows means and Student's
    \(t\)-based 95\% confidence intervals. } \label{fig:binary-sim}
\end{figure}

\subsection{Comparisons of Normalization Methods on Real Data}
\label{sec:normalization-tuning}

In this section we present an extended and modified version of the experiment in
\Cref{sec:experiments-predictive-performance}, by considering several additional datasets,
including datasets with binary responses to which we have fit \(\ell_1\) and
\(\ell_2\)-regularized logistic regression. Instead of only considering parameterization
over \(\delta\), we also extend the benchmarks to cover additional normalization types. For
each of the datasets, we have fit the lasso and ridge for 10-times repeated 10-folds
cross-validation over a grid of \(\lambda\) and normalization method. In the case of the
method we call ``ours'', we have extended the grid over \(\delta\) as well, and return the
results for the \(\delta\) with lowest error.

The results are shown in \Cref{fig:method-comparison} and show that the optimal choice of
normalization depends on the dataset and type of model. Using our strategy, which
corresponds to standardization for continuous data and hyper-tuning over \(\delta\) attains
best results for most of the datasets, only ever slightly worse than the best performing
method. Among the methods, \(\ell_1\) normalization seems to perform poorly in general.
Given the fact that that it corresponds to variance-scaling for binary data, which we have
shown results in considerable variance, this is not particularly surprising.

\begin{figure*}[htpb]
  \centering
  \subfigure[%
    Results for the lasso
  ]{\includegraphics[]{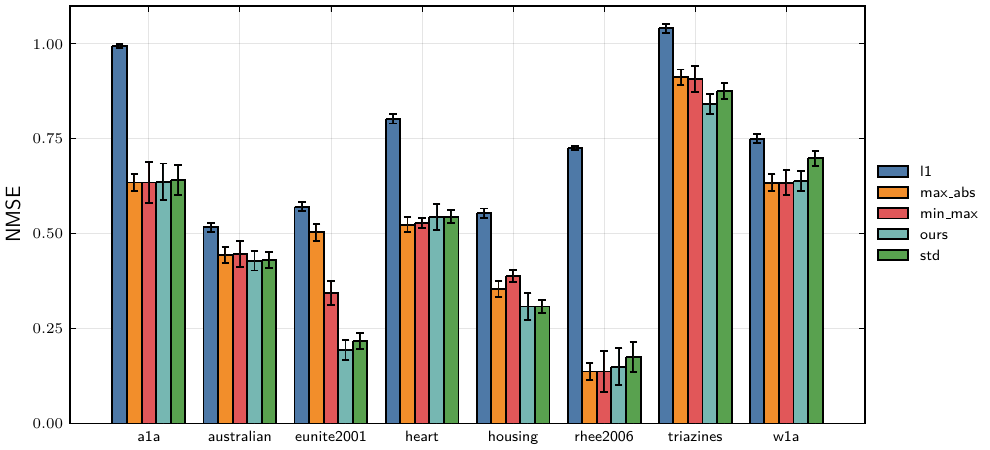} }
  \subfigure[%
    Results for ridge regression
  ]{\includegraphics[]{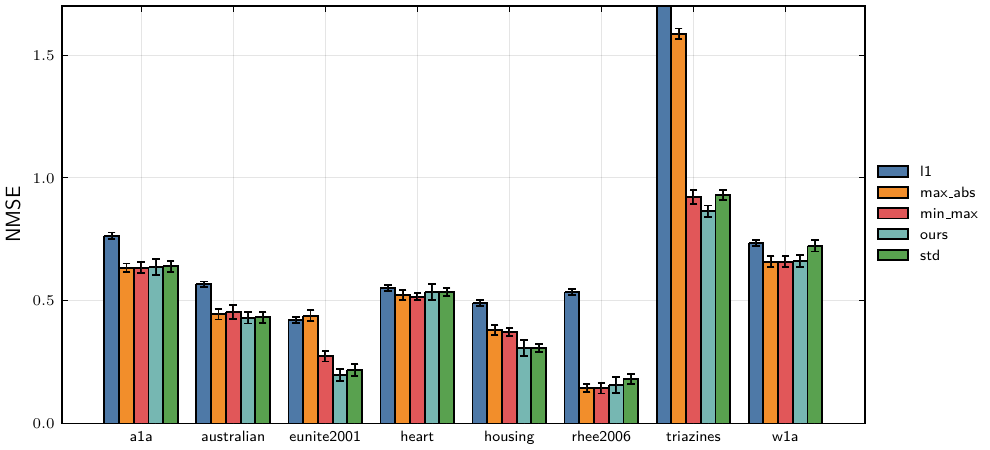}}
  \caption{%
    Cross-validation error from 10-times repeated 10-folds cross validation for the lasso and ridge
    and various datasets and normalization strategies. The error is normalized mean-squared
    error (NMSE). In the case of datasets \data{a1a}, \data{w1a}, \data{heart}, \data{w1a},
    and \data{australian}, we have fit regularized logistic regression and
    otherwise regularized linear regression. The error bars show 95\% confidence intervals.
  }
  \label{fig:method-comparison}
\end{figure*}

\subsection{Dichotomization and Feature Selection}
\label{sec:dichotomization}

In this experiment, we attempt to study the effect of normalization choice in feature
selection on datasets with dichotomized features and the following dichotomization scheme
to convert the continuous features to binary.

Each variable was converted to a binary feature by comparing against threshold based on
either historical standards, regulatory guidelines, or domain-specific knowledge of urban
housing factors in the 1970s~(\Cref{tab:dichotomization}).

\begin{table}[htbp]
  \centering
  \caption{
    Dichotomization procedure and linear regression coefficients for the
    features of the Boston housing dataset used in the experiment in
    \Cref{sec:dichotomization}.
  }
  \label{tab:dichotomization}
  \begin{tabular}{
      l
      p{9.5cm}
      S[table-format=-1.2,round-mode=places,round-precision=2]
      S[table-format=0.2,round-mode=places,round-precision=2]
    }
    \toprule
    Feature             & Dichotomization Rule                                                                         & $\hat{\bm{\beta}}_j^\text{OLS}$ & $q_j$     \\
    \midrule
    Crime Rate          & Above/below national U.S. average crime rate (1970--1971)                                    & -1.872296271668692              & 0.8992094 \\
    Zoning              & Presence/absence of large lot zoning (residential lots over 25,000 sq.ft)                    & 1.956883402813590               & 0.2648221 \\
    Business            & Residential vs industrial areas (below/above 10\% non-retail business acres)                 & -2.41422163254033               & 0.4664031 \\
    Charles River       & Original binary feature (borders river or not)                                               & 5.118087353331369               & 0.0691699 \\
    NOx Concentration   & Above/below EPA air quality standard for NOx (53 parts per billion)                          & -2.607920261778666              & 0.5177865 \\
    Rooms               & Above/below typical family dwelling size (more/less than 6 rooms)                            & 3.486474827676998               & 0.6581027 \\
    Age                 & Newer vs historic housing (less/more than 50\% built before 1940)                            & -1.101306211297333              & 0.7094861 \\
    Distance            & Close vs far from employment centers (less/more than 5 miles)                                & -4.762540742210745              & 0.2667984 \\
    Highway Access      & Limited vs good highway access (accessibility index below/above 20)                          & 1.250197827256583               & 0.2608695 \\
    Property Tax        & Below/above Massachusetts average property tax rate (approximately \$12 per \$1000 in 1970s) & -7.347485535327863              & 0.9664031 \\
    Pupil-Teacher Ratio & Below/above recommended educational value (16 students per teacher)                          & -6.567220592441502              & 0.8320158 \\
    Demographics        & More/less diverse population (above/below 85\% white)                                        & 3.245798334030099               & 0.9486166 \\
    Lower Status        & Middle-class vs lower-income areas (below/above 15\% lower status population)                & -7.240796759180348              & 0.3201581 \\
    \bottomrule
  \end{tabular}
\end{table}

Having dichotomized the data, we then first fit a standard linear regression model to the
data, and use this as a proxy for feature importance. Then, for three types of
normalization: \(\ell_1\)-normalization (\(\delta=1\)), standardization, and min--max
normalization, we fit a lasso model to the data and compute ranks of the features by
checking at which point they enter the model.

Finally, we compare the ranks of the features in the lasso model to the ranks of the
coefficients from the standard linear regression model, using the latter as the reference.
We use four different metrics to compare the ranks: Spearman's rank correlation, Kendall's
\(\tau\), mean absolute difference, and normalized discounted cumulative gain (NDCG). The
results are presented in \Cref{tab:method_comparison}. The results show that the
normalization method using \(\ell_1\)-normalization (\(\delta = 1\)) best corresponds to
the ranks of the linear regression coefficients.

\begin{table}[htbp]
  \centering
  \caption{
    Comparison between ranks of ordinary least-squares coefficients and ranks
    given by the order of model entry along the lasso path for the Boston
    housing dataset. The metrics used are Spearman's and Kendall's rank
    correlations, normalized discounted cumulative gain (NDCG), and mean
    absolute difference (MAD). Best values are marked in blod face. For all
    measures except MAD, higher values are better.
  }%
  \label{tab:method_comparison}
  \begin{tabular}{l S[table-format=1.4,detect-weight] S[table-format=1.4] S[table-format=1.4]}
    \toprule
    Metric   & {\(\ell_1\)-Normalization} & {Standardization} & {Min--Max/Max--Abs} \\
    \midrule
    Spearman & \best 0.7308               & 0.5714            & 0.5                 \\
    Kendall  & \best 0.5128               & 0.4359            & 0.3846              \\
    MAD      & \best 2.0                  & 2.7692            & 3.0769              \\
    NDCG     & \best 0.9515               & 0.9351            & 0.9186              \\
    \bottomrule
  \end{tabular}
\end{table}

\section{Summary of Data Sets}\label{sec:data-summary}

In \Cref{tab:dataset-info} we summarize the datasets we use in our paper.

\begin{table}
  \centering
  \caption{Details of the real datasets used in the experiments. The median \(q\) value
    refers to the median of the proportion of ones for the binary features in the data. Note that in the case of \data{housing}, there is
    only a single binary feature.}
  \label{tab:dataset-info}
  \small
  \begin{tabular}{
      l
      S[table-format=4.0]
      S[table-format=4.0]
      l
      l
      S[table-format=1.3,round-mode=places,round-precision=3]
      p{6cm}
    }
    \toprule
    Dataset           & {\(n\)} & {\(p\)} & Response   & Design     & {Median \(q\)} & Description                                                                                                                                                                                                                                                                                 \\
    \midrule

    \data{eunite2001} & 336     & 16      & continuous & mixed      & 0.856143       & Mid-term load forecasting competition dataset (EUNITE 2001) using National Taiwan University's winning approach. Contains 15 features including 7-day historical loads (scaled) with winter-only training data from 1997-1998 to predict January 1999 daily maximum loads \citep{chen2004}. \\

    \addlinespace
    \data{housing}    & 506     & 13      & continuous & mixed      & 0.93083        & Boston housing dataset with information about housing in the Boston area. Response is median value of owner-occupied homes in \$1000s \citep{harrison1978}.                                                                                                                                 \\

    \addlinespace
    \data{triazines}  & 186     & 60      & continuous & mixed      & 0.973118       & Pharmaceutical dataset relating molecular structures of triazine derivatives to their ability to inhibit dihydrofolate reductase \citep{hirst1994,king1995}.                                                                                                                                \\

    \addlinespace
    \data{rhee2006}   & 842     & 361     & continuous & binary     & 0.995249       & HIV-1 drug resistance data with protease and reverse transcriptase mutations as features. Response measures in vitro susceptibility to antiretroviral drugs \citep{rhee2006}.                                                                                                               \\

    \addlinespace
    \data{leukemia}   & 38      & 7129    & binary     & continuous &                & Gene expression data for leukemia patients. Classifies between acute myeloid leukemia (AML) and acute lymphoblastic leukemia (ALL) \citep{golub1999}.                                                                                                                                       \\

    \addlinespace
    \data{australian} & 690     & 14      & binary     & continuous & 0.557246       & Credit approval dataset, originally from the StaLog database. The task is to predict credit approval using a number of different features~\citep{quinlan1987,henery1992}.                                                                                                                  \\

    \addlinespace
    \data{heart}      & 270     & 13      & binary     & mixed      & 0.677778       & Heart disease dataset, originally from the StatLog database. The task is to predict the presence of heart disease from a number of features that have already been selected from a larger set of features~\citep{henery1992}.                                                              \\

    \addlinespace
    \data{a1a}        & 1605    & 123     & binary     & binary     & 0.970093       & Subset of Adult dataset derived from census data. Predicts whether income exceeds \$50,000/year based on census information \citep{becker1996,platt1998}.                                                                                                                                   \\

    \addlinespace
    \data{w1a}        & 2477    & 300     & binary     & binary     & 0.976181       & Derived from web page data, classifying whether pages belong to specific categories \citep{platt1998}.                                                                                                                                                                                      \\

    \bottomrule
  \end{tabular}
\end{table}

We also visualize the distribution of class balance among all the binary features in
\Cref{fig:data-hist-q}. We note that the class imbalance for many of these datasets is
quite severe.

\begin{figure}[htpb]
  \centering
  \includegraphics[]{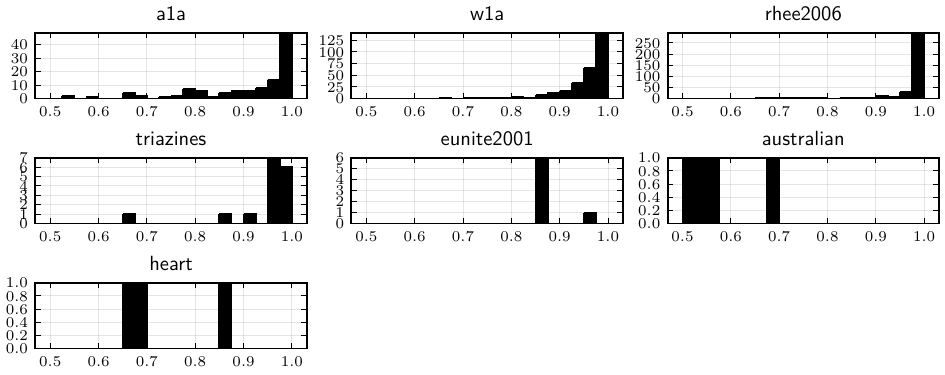}
  \caption{%
    Histograms over the distribution of \(q\) (class balance, that is, the
    proportion of ones) for the binary features in each of the datasets
    used in the paper.
  }
  \label{fig:data-hist-q}
\end{figure}

\end{document}